\documentclass{article}

\usepackage{microtype}
\usepackage{graphicx}
\usepackage{subfigure}
\usepackage{booktabs}
\usepackage{hyperref}

\usepackage{url}
\usepackage{wrapfig}
\usepackage{amsmath}
\usepackage{amssymb}
\usepackage{mathtools}
\usepackage{amsthm}
\usepackage{thmtools, thm-restate}
\usepackage{multirow}
\usepackage{enumitem}

\usepackage[preprint]{icml2026}

\theoremstyle{plain}
\newtheorem{theorem}{Theorem}[section]

\newtheorem{lemma}[theorem]{Lemma}

\theoremstyle{definition}

\theoremstyle{remark}
\newtheorem{remark}[theorem]{Remark}
\usepackage{xspace}

\newcommand{\method}{Soft FB\xspace}

\icmltitlerunning{Soft Forward-Backward Representations}

\newcommand{\icmlEqualSupervision}{\textsuperscript{*}Equal senior authorship }

\begin{document}

\twocolumn[
\icmltitle{\textit{Soft} Forward-Backward Representations for \\ Zero-shot Reinforcement Learning with General Utilities}

\icmlsetsymbol{equal}{*}

\begin{icmlauthorlist}
\icmlauthor{Marco Bagatella}{eth,mpi}
\icmlauthor{Thomas Rupf}{eth}
\icmlauthor{Georg Martius}{equal,unitue}
\icmlauthor{Andreas Krause}{equal,eth}
\end{icmlauthorlist}

\icmlaffiliation{eth}{ETH Zurich, Zurich, Switzerland}
\icmlaffiliation{mpi}{Max Planck Institute for Intelligent Systems, Tubingen, Germany}
\icmlaffiliation{unitue}{University of Tubingen, Tubingen, Germany}

\icmlcorrespondingauthor{Marco Bagatella}{mbagatella@ethz.ch}

\icmlkeywords{Forward-Backward, Successor Features, Unsupervised Reinforcement Learning, General Utility Reinforcement Learning, Convex Reinforcement Learning}

\vskip 0.3in
]

\printAffiliationsAndNotice{\icmlEqualSupervision}

\begin{abstract}
Recent advancements in zero-shot reinforcement learning (RL) have facilitated the extraction of diverse behaviors from unlabeled, offline data sources.
In particular, forward-backward algorithms (FB) can retrieve a family of policies that can approximately solve any standard RL problem (with additive rewards, linear in the occupancy measure), given sufficient capacity.
While retaining zero-shot properties, we tackle the greater problem class of {\em RL with general utilities}, in which the objective is an arbitrary differentiable function of the occupancy measure.
This setting is strictly more expressive, capturing tasks such as distribution matching or pure exploration, which may not be reduced to additive rewards.
We show that this additional complexity can be captured by a novel, maximum entropy (\textit{soft}) variant of the forward-backward algorithm, which recovers a family of stochastic policies from offline data.
When coupled with zero-order search over compact policy embeddings, this algorithm can sidestep iterative optimization schemes, and optimizes general utilities directly at test-time.
Across both didactic and high-dimensional experiments, we demonstrate that our method retains favorable properties of FB algorithms, while also extending their range to more general RL problems.
\end{abstract}

\section{Introduction}

\begin{figure}
    \centering
    \vspace{-4mm}
    \includegraphics[width=\columnwidth]{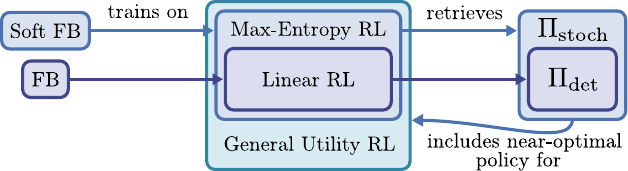}
    \vspace{-5mm}
    \caption{
    We propose \method, a soft version of the Forward-Backward algorithm which solves maximum entropy RL instances to retrieve a richer set of stochastic policies, and searches them to optimize general utilities at test-time.}
    \label{fig:teaser}
    \vspace{-7mm}
\end{figure}

Departing from the strict necessity of an explicit reward signal at training time, zero-shot RL has recently received increasing attention.
Although several definitions have been brought forward, they largely agree: zero-shot methods may use significant amounts of compute in an initial, unsupervised pretraining phase, but should be capable of producing near-optimal behavior with minimal computation for any reward specified at test-time \citep{touati2023does}.
According to this criterion, goal-reaching RL \citep{andrychowicz2017hidsight, eysenbach2022contrastive} may be seen as a practical instance of zero-shot RL, which is, however, limited to indicator reward functions.
Recent works further expand the goal-reaching setting to allow for the specification of more complex behavior \citep{frans2024unsupervised, sikchi2024rl, agarwal2024proto}.
Among them, forward-backward representations \citep{touati2021learning} have been proposed as a zero-shot method capable of solving for arbitrary (Markov) reward functions at test-time.
Although some have speculated that this family of reward functions is sufficient to extract any behavior of interest \citep{silver2021reward}, its expressiveness remains limited \citep{kumar2022policy}. \looseness -1

In this work, we turn towards a broader class of RL problems: while Markov rewards lead to a linear objective in the policy's occupancy, we aim to optimize arbitrary differentiable functions of the policy's occupancy, known as General Utilities (GU,  \citealp{zhang2020variational, zhang2021convergence, kumar2022policy, barakat2023reinforcement}).
This increased scope includes interesting problem instances, such as pure exploration \citep{hazan2019provably}, active learning \citep{mutny2023active} and learning from observations \citep{torabi2019recent}, which are generally beyond the reach of linear RL algorithms.
A large body of work has explored similar settings, with particular focus on convex \citep{zahavy2021reward, mutti2022challenging} or submodular objectives \citep{de2024global}. Generally, the resulting practical algorithms rely on online learning through semi-gradient methods \citep{zahavy2021reward, de2024global}.
Ad-hoc algorithms have been proposed for specific instances of General Utilities (e.g.,  \citet{ho2016generative, hazan2019provably, ma2022offline, bolland2024off}), but they are not designed to generalize beyond their respective domain.

In this work, we introduce \textbf{S}oft \textbf{FB} (SFB), a practical zero-shot algorithm for RL with GUs.
Building upon the forward-backward framework \citep{touati2021learning}, we introduce a soft variant \citep{ziebart2008maximum, haarnoja2018soft, hunt2019composing} which retrieves a family of stochastic policies, approximating all solutions to maximum entropy RL problems.
Interestingly, given sufficient capacity and exact training, the set of policies retrieved through entropy-regularized linear RL can be shown to contain a near-optimal Markov policy for each differentiable GU.
Thus, one may optimize GUs zero-shot by simply (i) learning maximum entropy policies and (ii) searching among them.
The first step bypasses the challenge of parameterizing GU instances; while the second step remains generally intractable, we find that it can be solved approximately and efficiently through zero-order search over low-dimensional policy representations.
Rather than running a dedicated per-objective optimization loop, we therefore amortize computation in an unsupervised pre-training phase, and optimize GUs through test-time search.
Furthermore, unlike existing Frank-Wolfe schemes \citep{hazan2019provably}, the solution retrieved is not a mixture of policies, but a single, Markov policy.
As a result, our method can be seen both as a generalization of existing zero-shot frameworks, or as a practical offline algorithm for optimizing GUs.

We demonstrate the flexibility of the method in illustrative continuous environments, in which Soft FB can solve several zero-shot RL problems that are out-of-reach for existing, linear methods.
We then extend this evaluation to established, complex zero-shot benchmarks, demonstrating scalability and studying the impact of entropy regularization for standard, linear tasks. Our contributions can be summarized as follows:
\begin{itemize}[noitemsep]
    \item We propose \method, a \textit{soft} forward-backward algorithm, capable of retrieving stochastic policies and optimizing GUs zero-shot.
    \item We provide formal guarantees for the expressiveness of our method.
    \item We present a thorough empirical evaluation of our method, demonstrating its ability to capture a richer class of policies while retaining the desirable properties of forward-backward algorithms.
\end{itemize}

After establishing essential notation and motivations in Sections \ref{sec:background} and \ref{sec:motivation}, we present and analyze our algorithm in Section \ref{sec:method}. Experiments, related works and a broad discussion follow in Section \ref{sec:exp}, \ref{sec:related} and \ref{sec:conclusion}, respectively.

\section{Background}
\label{sec:background}

In the context of sequential decision making, a general way to describe an environment is through a reward-free MDP $\mathcal{M} = (\mathcal{S}, \mathcal{A}, P, \mu_0, \gamma)$, where $\mathcal{S}$ and $\mathcal{A}$ are potentially continuous state and actions spaces, $P: \mathcal{S} \times \mathcal{A} \to \Delta(\mathcal{S})$ is a stochastic transition kernel\footnote{$\Delta(\cdot)$ denotes the space of distributions over $\cdot$}, $\mu_0 \in \Delta(\mathcal{S})$ is an initial state distribution, and $\gamma \in (0,1)$ is a discount factor.
Additionally, the agent's behavior may be described by a Markov policy, that is a state-conditional action distribution $\pi: \mathcal{S} \to \Delta(\mathcal{A})$.
These simple entities induce a discounted occupancy measure over states-action pairs, also known as the \textit{successor measure} of policy $\pi$:
\begin{align*}
M^\pi(s,a,X) \hspace{-0.5ex} = \hspace{-0.5ex} (1-\gamma)\hspace{-0.5ex} \sum_{t\geq 0}\hspace{0ex} \gamma^{t} \text{Pr}((s_t,a_t) \hspace{-0.5ex} \in \hspace{-0.5ex} X \hspace{-0.5ex} \mid  \hspace{-0.5ex} s_0 \hspace{-0.5ex} = \hspace{-0.5ex} s,a_0 \hspace{-0.5ex} = \hspace{-0.5ex} a)
\end{align*}
where $X \subseteq \mathcal{S \times A}$ and $\text{Pr}(\cdot)$ is the visitation likelihood under policy $\pi$ and dynamics $P$.
If $\mathcal{S}$ and $\mathcal{A}$ are both discrete, $M^\pi$ may be directly represented as a $(|\mathcal{S}||\mathcal{A}| \times |\mathcal{S}||\mathcal{A}|)$ stochastic matrix.
Intuitively, each element of the matrix would track the discounted cumulative likelihood of visiting a given state-action pair, starting from another state-action pair.
While our algorithm generalizes to continuous spaces, we will consider the discrete case in our formal derivations.\looseness -1

We further note that the successor measure might instead track state-only visitations. By overloading $X \in \mathcal{S}$,
\begin{align*}
M_\mathcal{S}^\pi(s,a,X) \hspace{-0.5ex} & = \hspace{-0.5ex} (1-\gamma)\hspace{-0.5ex} \sum_{t\geq 0}\hspace{0ex} \gamma^{t} \text{Pr}(s_t \hspace{-0.5ex} \in \hspace{-0.5ex} X \hspace{-0.5ex} \mid  \hspace{-0.5ex} s_0 \hspace{-0.5ex} = \hspace{-0.5ex} s,a_0 \hspace{-0.5ex} = \hspace{-0.5ex} a) \\
& = \hspace{-0.5ex} M^\pi(s,a,X \times \mathcal{A}).
\end{align*}
Due to a particular choice of parameterization, our practical algorithm will estimate this object directly.
However, samples $(s, a) \sim M^\pi$ can be easily obtained by sampling $s \sim M^\pi_\mathcal{S}$ and $a \sim \pi(\cdot|s)$, as we describe in Section \ref{sec:inference}.
We will finally use $\mathrm{M}^\pi$ and $\mathrm{M}^\pi_\mathcal{S}$ to denote successor measures when marginalized over the initial state distribution $\mu_0$, e.g. $\mathrm{M}^\pi(X) = \mathop{\mathbb{E}}_{s_0 \sim \mu_0, a_0 \sim \pi(s_0)} M^\pi(s_0, a_0, X)$.

Many interesting RL problem instances can be directly defined as a function of the successor measure $M^\pi$, and classified according to the properties of the function.
Standard RL problems can be expressed as a discounted sum of Markov rewards: if the rewards are expressed as a vector $R$ the objective can be simply computed as $J^\pi_\text{lin} = \langle \mathrm{M}^\pi , R \rangle$, and this instance is thus referred to as \textit{Linear} RL.
Another instance that has received significant attention in recent years is that of \textit{Maximum Entropy} RL \citep{ziebart2008maximum, haarnoja2018soft}, which adds an entropy term to the linear objective: $J^\pi_\mathcal{H} = \langle \mathrm{M}^\pi , R + \mathcal{H}^\pi\rangle$, where $\mathcal{H}^\pi$ is a vector which contains the entropy of the policy $\mathcal{H}^\pi(s, a) = \mathcal{H}(\pi(\cdot|s))$ for each state-action pair $(s,a) \in \mathcal{S} \times \mathcal{A}$.
Because of this policy-dependent entropy term, maximum entropy RL is not a linear problem in the occupancy, but rather an instance of \textit{Convex} RL \citep{mutti2022challenging}.
In general, Convex RL problems can be expressed as $J^\pi_{\text{con}}=f(\mathrm{M}^\pi)$, where $f(\cdot)$ is an arbitrary convex function.
Finally, all these instances are encompassed by \textit{General Utility} RL \citep{zhang2020variational}, which aims at optimizing a general differentiable scalar objective $J^\pi_{\text{GU}}=f(\mathrm{M}^\pi)$.

\section{Forward-Backward Representations and General Utilities}
\label{sec:motivation}

Having established common terminology, we will now analyze the forward-backward algorithm for zero-shot reinforcement learning \citep{touati2021learning, blier2021learning} and pinpoint its limitations motivating this work.
At its core, the FB framework introduces a family of parameterized policies $\{\pi_z\}_{z \in \mathbb{R}^d}$, each inducing an occupancy $M^z := M^{\pi_z}$. Crucially, each occupancy undergoes a specific low-rank decomposition 
\begin{equation}
M^z=F_z^\top B,    
\label{eq:low_rank_decomposition}
\end{equation}
where both $F_z$ and $B$ are $(d \times |\mathcal{S}||\mathcal{A}|)$ matrices, and may be called the \textit{forward} and \textit{backward} representation matrices.
If the decomposition holds, then for a $|\mathcal{S}||\mathcal{A}|$-dimensional reward vector $R$, the Q-function for a policy $\pi_z$ can be simply expressed as
\begin{equation}
Q^z = M^z R = F_z^\top BR \stackrel{BR = z'}{:=} F_z^\top z' ,
\end{equation}
where $B$ can be reinterpreted as a projection from the space of reward vectors to a $d$-dimensional reward embedding $z' := BR$.
Finally, the policy $\pi_z$ is enforced to be optimal with respect to the discounted sum of rewards encoded by its own reward embedding, that is
\begin{equation}
\pi_z \in \mathop{\text{argmax }}_\mathcal{A} F_z^\top z, 
\label{eq:greedy_policy}
\end{equation}
leading to the following formal result:
\begin{theorem} {\normalfont \citet{touati2021learning}} For an arbitrary bounded reward vector $R \in \mathbb{R}^{|\mathcal{S}||\mathcal{A}|}$, if both Equations \ref{eq:low_rank_decomposition} and \ref{eq:greedy_policy} hold for all $z \in \mathcal{Z}$, $\pi_{BR}$ is optimal with respect to $R$: $M^{\pi_{BR}}R = \max_{\pi} M^\pi R$.
\end{theorem}
While this theorem guarantees the retrieval of a solution for each \textit{linear} RL problem, the policy extraction objective in Equation \ref{eq:greedy_policy} (or its empirical counterparts, as we demonstrate in Appendix \ref{app:collapse}) may simply produce deterministic policies. This is sufficient for any linear RL instance, as they each admit an optimal deterministic policy \citep{sutton1998reinforcement}, but remains generally suboptimal for GUs.
This is the main insight motivating this work: the set of policies retrieved by FB may not contain the solution to non-linear RL problems.
\begin{remark}
    Let $\Pi=\{\pi: \mathcal{S} \to \Delta(\mathcal{A})\}$ be the set of all Markov policies. There exist an MDP $\mathcal{M}$ and a scalar function $f$ of occupancy measures and a set of policies $\Pi_z = \{\pi_z\}_{z \in \mathcal{Z}}$ such that (i) $\pi_z$ satisfies Equations \ref{eq:low_rank_decomposition} and \ref{eq:greedy_policy} for all $z \in \mathcal{Z}$, and (ii) $\max_{\pi \in \Pi} f(\mathrm{M}^\pi) > \max_{\pi \in \Pi_z} f(\mathrm{M}^\pi)$.
\end{remark}
\textbf{Counterexample \;} As a straightforward counter-example, it suffices to consider an MDP with a single state $\mathcal{S}=\{s_0\}$ and two actions $\mathcal{A}=\{a_0, a_1\}$, and a convex RL objective such as \textit{pure exploration} over states and actions, i.e. $J^\pi = f(\mathrm{M}^\pi) = \mathcal{H}(\mathrm{M}^\pi)$, where $\mathcal{H}(\cdot)$ denotes entropy over state-action pairs.
For each reward function in $\mathbb{R}^2$, forward and backward representations satisfying both Equations \ref{eq:low_rank_decomposition} and \ref{eq:greedy_policy} exist, but may not retrieve the (optimal) uniform policy (see Appendix \ref{app:counterexample}).
Following this motivation, we propose a variant of the FB algorithm which recovers a richer class of policies, and may provably optimize all general utilities.

\section{\textit{Soft} Forward-Backward Representations}
\label{sec:method}

\subsection{Core algorithm}
As for the linear case, we start by introducing a family of parameterized policies $\{\pi_z\}_{z \in \mathbb{R}^d}$, and decompose their occupancy as $M^z=F_z^\top B$, thus enforcing Equation \ref{eq:low_rank_decomposition} again.
We however introduce an entropy regularization term to encourage stochastic behavior \citep{haarnoja2018soft}.
The action-state value function for a policy $\pi_z$ is then expressed as \looseness -1
\begin{align}
Q_\text{soft}^z & = M^z (R + \mathcal{H}_{\pi_z}) \nonumber \\
& = F_z^\top BR + M^z \mathcal{H}_{\pi_z} \nonumber \\
& \stackrel{z'=BR}{=} F_z^\top z' + M^z \mathcal{H}_{\pi_z},
\end{align}
where $\mathcal{H}_{\pi_z} \in \mathbb{R}_+^{|\mathcal{S}||\mathcal{A}|}$ contains the policy's entropy at each state, and is a convex term in the occupancy measure.
Conforming to soft policy improvement \citep{haarnoja2018soft}, each policy $\pi_z$ is then defined with respect to its regularized action-value function:
\begin{equation}
\pi_z \propto \text{exp}(F_z^\top z + M^z \mathcal{H}_{\pi_z}),
\label{eq:soft_policy}
\end{equation}
Equation \ref{eq:low_rank_decomposition}, which remains unchanged, and Equation \ref{eq:soft_policy} are the core of the Soft FB algorithm.
The introduction of entropy regularization alters the set of policies that are retrieved during training to include all policies with full support.
Crucially, as we will show formally in Section \ref{sec:enough}, training on maximum entropy RL instances is sufficient to capture $\epsilon$-optimal solutions to general utilities.
We will later describe a practical algorithm for estimating soft forward and backward representations from continuous data, in combination with sample-based objectives and function approximation (see Section \ref{sec:practical}).

\subsection{Guarantees}
\label{sec:enough}

Due to the introduction of entropy regularization, we can show that Soft FB retrieves the optimal policy among all Markov policies $\Pi$ for each maximum entropy RL instance.
\begin{restatable}{theorem}{maxent}
For an arbitrary bounded reward vector $R \in \mathbb{R}^{|\mathcal{S}||\mathcal{A}|}$
, if both Equations \ref{eq:low_rank_decomposition} and \ref{eq:soft_policy} hold for all $z \in \mathcal{Z}$, $\pi_{BR}$ is the optimal maximum entropy policy with respect to $R$: $M^{\pi_{BR}}(R+ \mathcal{H}_{\pi_{BR}}) = \max_{\pi \in \Pi} M^\pi (R+ \mathcal{H}_{\pi})$.\looseness -1
\label{th:maxent}
\end{restatable}
Moreover, while only trained for solving maximum entropy instances, we can show that the set of policies retrieved by Soft FB includes arbitrarily good solutions to a greater class of problems: namely, it optimizes \textit{any} general utility.
\begin{restatable}{theorem}{generalutility}
    Let $f$ be an arbitrary differentiable function of occupancy measures, and $\tilde \Pi_z$ be the set of policies retrieved by Soft FB.
    If both Equations \ref{eq:low_rank_decomposition} and \ref{eq:soft_policy} hold for all $z \in \mathcal{Z}$, for any $\epsilon>0$ there exists a reward embedding $z' \in \mathcal{Z}$ such that $\pi_{z'} \in \tilde \Pi_{z}$ and $\max_{\pi \in \Pi} f(\mathrm{M}^\pi) - f(\mathrm{M}^{\pi_{z'}}) < \epsilon$.
    \label{cor:general}
    \vspace{-4mm}
    \label{th:generalutility}
\end{restatable}
We remark that this is an existence result; a good policy is in practice recovered via search and successor-measure estimation, as described in Section \ref{sec:inference}.
We refer to Appendix \ref{app:proofs} for proofs of these two statements and further formal remarks. The latter has important practical consequences: while parameterizing general utilities is far from straightforward, maximum entropy instances can be easily parameterized, as we discuss in Section \ref{sec:practical}. As a result, we can directly optimize general utilities at test-time, while only dealing with tractable maximum entropy objectives at training time.\looseness -1

\subsection{Practical algorithm}
\label{sec:practical}

An algorithmic instantiation of \method in potentially continuous spaces needs to address two points. First, the embedding $z \in \mathcal{Z}$ solving each General RL problem as described so far would lie in $\mathbb{R}^d$, which is arbitrarily large and impractical to search. We will show that this can be easily addressed through a simple reparameterization. Second, we will introduce sample-based objectives that may be used to train function approximators over continuous state and action spaces, and approximately enforce Equations \ref{eq:low_rank_decomposition} and \ref{eq:soft_policy}.\looseness -1

\textbf{Reparameterization \;} \method reduces policy optimization to search over stochastic policies parameterized by vectors in $\mathcal{Z}=\mathbb{R}^d$.
In principle, every stochastic policy may be retrieved for some $z \in \mathcal{Z}$; however, this embedding could lie anywhere in $\mathbb{R}^d$. In fact, deterministic policies are only retrieved by embeddings whose norm approaches infinity (such that $F_z^\top z$ completely outweighs the entropy regularization term in Equation \ref{eq:soft_policy}).
Fortunately, recalling that optimal policies are invariant to linear transformations of Q-functions, we can find an alternative parameterization for which $z$ may be sampled from a bounded space. Starting from $z \in \mathcal{Z}$, we observe that 
\begin{align}
    Q^z & = F_z^\top z + M^z \mathcal{H}_{\pi_z} \nonumber \\
    & \propto \frac{1}{\|z\| + 1} (F_z^\top z + M^z \mathcal{H}_{\pi_z}) \nonumber \\
    & \stackrel{z' := \frac{z}{\|z\| + 1}}{=} F_z^\top z' + (1-\|z'\|) M^z \mathcal{H}_{\pi_z}.
    \label{eq:reparam}
\end{align}
Therefore, as $0 \leq \| z' \| < 1$, we can now map all maximum entropy RL instances to embeddings $z \in \mathcal{Z}$ sampled from the volume of a $d$-dimensional hypersphere \footnote{\citet{touati2021learning} operate over normalized embeddings $z$, thus sampling them from the \textit{surface} of the hypersphere.}.
This allows an intuitive geometric interpretation: the origin of $\mathbb{R}^d$ is associated with a uniform policy (as $z$ is the null vector, and the Q-function equates the entropy bonus alone), and embeddings on the surface parameterize the family of deterministic policies that standard FB algorithms would return (see Figure \ref{fig:hypersphere}). \looseness -1

\begin{figure}
    \includegraphics[width=0.8 \linewidth]{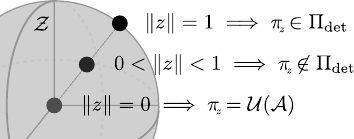}
    \vspace{0mm}
    \caption{Geometric interpretation of $z$ after reparameterization: the stochasticity of $\pi_z$ grows with $\|z\|$.}
    \label{fig:hypersphere}
    \vspace{-4mm}
\end{figure}

\begin{figure*}
    \centering
    \includegraphics[width=\linewidth]{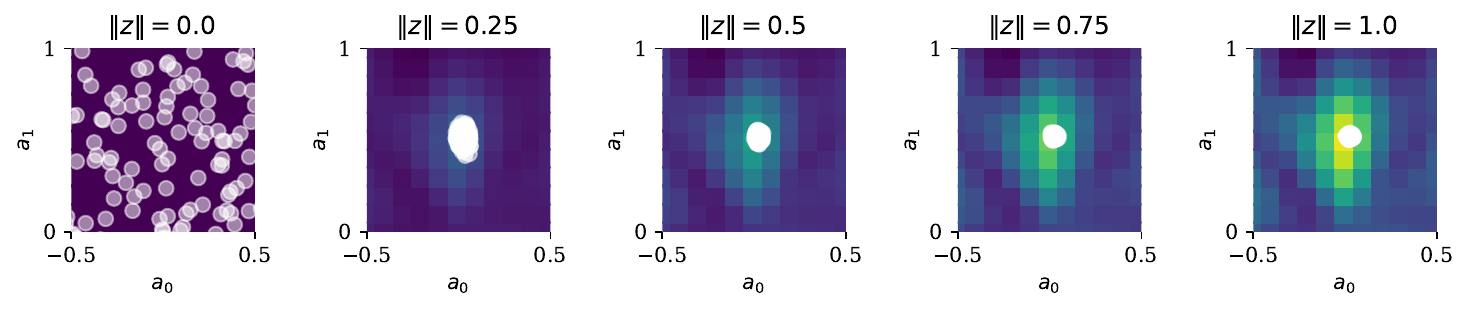}
    \vspace{-6mm}
    \caption{Qualitative evaluation of \method in a didactic environment. White dots are samples from policies $\pi_z$ over a 2D actions space, and the color map represents learned unregularized Q-values $Q_R^z$ for each action ($F_\theta(s_0, a, z)^\top z$). From left to right, we infer task embeddings $z$ for a goal-reaching task, and scale them linearly. The policies conditioned on $z$ become more deterministic as its norm increases. The same visualization for FB can be found in Appendix \ref{app:collapse}.}
    \label{fig:qualitative}
    \vspace{-3mm}
\end{figure*}

\textbf{Training objectives \;} A practical instantiation of \method for continuous spaces can leverage the framework proposed by \citet{touati2021learning}: we can approximate each entry of $M^z$ by the dot product between a forward and backward embedding, produced by function approximators $F_\theta: \mathcal{S} \times \mathcal{A} \times \mathcal{Z} \to \mathcal{Z}$ and $B_\phi: \mathcal{S} \to \mathcal{Z}$ \footnote{Following prior work \citep{touati2021learning}, our practical instantiation models the successor measure $M^\pi_\mathcal{S}$ over states only.}. The decomposition from Equation \ref{eq:low_rank_decomposition} is thus generalized as an equality over measures:
$M^z(s, a, ds') \approx F_\theta(s, a, z)B_\phi(s')\rho(s')$,
where $\rho(\cdot)$ represents the data distribution.
We can then adopt a standard objective for training forward and backwards representations, minimizing Bellman residuals by sampling from the data distribution $\rho$ (i.e., an offline dataset of transitions) \citep{blier2021learning}:
\begin{align}
    \mathcal{L}_\text{FB}(\theta, \phi) & = \mathop{\mathbb{E}}_{\substack{z \sim \mathcal{U}(\mathcal{Z}),\; (s_t,a_t,s_{t+1}) \sim \rho \\ s' \sim \rho, \; a_{t+1} \sim \pi_z(s_{t+1})}} \Big[ \big( F_\theta(s_t, a_t, z)^\top B_\phi(s') \nonumber \\
    & - \gamma \bar F(s_{t+1}, a_{t+1}), z)^\top \bar B(s') \big)^2 \nonumber \\
    & - 2F_\theta(s_t, a_t, z)^\top B_\phi(s_{t+1}) \Big],
\end{align}
where $\bar F$ and $\bar B$ may be target networks \citep{fujimoto2018addressing}, and forward representations are averaged over twin networks. As in \citet{touati2021learning}, we employ an auxiliary orthonormalization loss over $B_\phi$.
In continuous spaces, the policy $\pi_z$ is also represented through function approximation, and may be expressed as $\pi_\psi(\cdot|s, z)$~\footnote{While \method is compatible with expressive policies with explicit likelihoods (e.g., normalizing flows), in our evaluation $\pi_\psi$ will be parameterized as a diagonal Gaussian.}.
However, the policy's learning rule needs to be altered significantly, in order to account for entropy regularization.
First, we'll split the maximum entropy action-state value function from Equation \ref{eq:reparam} in a reward-based component, and an entropy-based component: $Q^z(s,a)=Q_R^z(s,a) + (1 - \|z\|) Q_\mathcal{H}^z(s,a)$. The former is easily computed as $Q_R^z(s, a) \propto F_\theta(s, a, z)^\top z$, while the latter can be estimated by training a parameterized critic  $Q_\mathcal{H, \eta}(s,a, z)$ through the TD objective
\begin{align}
\label{eq:entropy_td}
    & \mathcal{L}_\mathcal{H}(\eta) = \mathop{\mathbb{E}}_{\substack{z \sim \mathcal{U}(\mathcal{Z}),\; (s_t,a_t,s_{t+1}) \sim \rho \\ a_{t+1} \sim \pi_z(s_{t+1})}} \Big[ \Big( Q_{\mathcal{H}, \eta}(s_t, a_t, z) \nonumber \\
    & - \gamma \big( \bar Q_{\mathcal{H}, \eta}(s_{t+1}, a_{t+1}, z) - \log \pi_\psi(a_{t+1}|s_{t+1}, z) \big) \Big)^2 \Big].
\end{align}
In any case, recalling Equations \ref{eq:soft_policy} and \ref{eq:reparam}, we update the policy parameters $\psi$ by minimizing the KL divergence to the soft policy implicitly defined by the regularized Q-function  \citep{haarnoja2018soft}:
\begin{equation}
    \mathcal{L}_{\pi}(\psi) = \mathop{\mathbb{E}}_{\substack{z \sim \mathcal{U(Z)} \\ s \sim \mathcal{\rho}}} D_{KL} \Bigg(\pi_\psi(\cdot|s, z), \frac{e^{\frac{Q^z(s, \cdot)}{1 - \|z\|} }}{Z(s)} \Bigg),
\end{equation}
with $Z(s) = \int_{\mathcal{A}}\exp \big((1-\|z\|)^{-1} Q^z(s, a) \big)da$. As $Z(s)$ is independent of $\pi_\psi$, we can redefine an equivalent loss up to policy-independent factors as 
\begin{align}
    \mathcal{L}_{\pi}(\psi) & = \mathop{\mathbb{E}}_{\substack{z \sim \mathcal{U(Z)}, \; s \sim \rho \\ a \sim \pi_\psi(\cdot|s, z)}} (1 - \|z\|) [\log \pi_\psi(a|s, z) \nonumber \\
    & - Q_{\mathcal{H}, \eta}(s, a, z)] - F_\theta(s, a, z)^\top z,
\end{align}
which is optimized with a standard reparameterization trick.

\subsection{Inference}
\label{sec:inference}

Soft FB returns a set of parameterized policies, $\{\pi_z\}_{z \in \mathcal{Z}}$; as shown in Corollary \ref{cor:general}, this family is rich enough to include a solution for a large class of problems. However, solving a specific downstream task requires searching the policy class $\{\pi_z\}_{z \in \mathcal{Z}}$.
Fortunately, policies are parameterized by a low-dimensional embedding $z \in \mathcal{Z}$, and the search process remains arguably efficient.
In practice, paralleling Generalized Policy Improvement techniques in \citet{farebrother2025temporal}, we resort to zero-order optimization. Given an objective $J^\pi = f(\mathrm{M}_\pi)$ in its analytical form at inference time, a good policy $\pi_{z^\star} \approx \text{argmax}_{\pi} f(\mathrm{M}_\pi)$ may be found through offline evaluation:
\begin{equation}
z^\star \approx \mathop{\text{argmax}}_{z\in \mathcal{Z}} f(\mathrm{\hat M}_z) ,
\end{equation}
where $\mathrm{\hat M}_z$ is a sample-based estimate of the measure induced by policy $\pi_z$.
As $\mathcal{Z}$ is relatively low-dimensional, zero-order, sampling-based optimization methods such as random shooting or CEM \citep{rubinstein1999cross} can be leveraged to find $z^\star$. Due to its simplicity, this is the solution we adopt in our empirical evaluation.
We will consider two ways to recover sample-based estimates $\mathrm{\hat M}_z$ of policy dependent measures: implicit and explicit.

\textbf{Implicit measure model} \; Approximately recovering samples from the measure is possible by sampling states from the pre-training buffer distribution $\rho$ with importance weights $F_z^\top B$: $s \sim \mathrm{\hat M}^{\pi_z}_\mathcal{S}$ where
\begin{equation}
        \mathrm{\hat M}^{\pi_z}_\mathcal{S}(s) \propto \mathop{\mathbb{E}}_{\substack{s_0 \sim \mu_0 \\ a_0 \sim \pi_\psi(s_0, z)}} F_\theta(s_0, a_0, z)^\top B_\phi(s) \rho(s),
\vspace{-1mm}
\end{equation}
and then drawing actions $a \sim \pi_\psi(\cdot|s, z)$, thus only leveraging networks we previously trained.

\textbf{Explicit measure model} \;
As the importance weights are often poorly approximated, we can alternatively leverage recent work \citep{farebrother2025temporal} and train a flow-based generative model of the successor measure, completely offline. In practice, this may produce more accurate samples, at the cost of increased compute during pre-training. We detail training objectives in Appendix \ref{app:tdflow}. The variant of algorithms relying on generative models for inference will be marked by a subscript (e.g., SFB$_{flow}$).

Beside zero-order optimization, for some specific objectives soft FB retains the closed-form solutions that characterize the FB framework. Linear RL problems can be solved by computing $z^\star=\frac{BR}{\|BR\|}$ \citep{touati2021learning}.
Moreover, a closed-form solution is also possible in the case of maximum entropy RL by computing the optimal embedding $z^\star$ for the corresponding linear RL instance, and rescaling it by $\frac{F_z^\top z}{F_z^\top z + 1}$.
Furthermore, all imitation learning approaches described in \citep{pirotta2024fast} remain possible.

\begin{figure*}

    \centering
    \begin{minipage}{0.5\linewidth}
    \includegraphics[width=\linewidth]{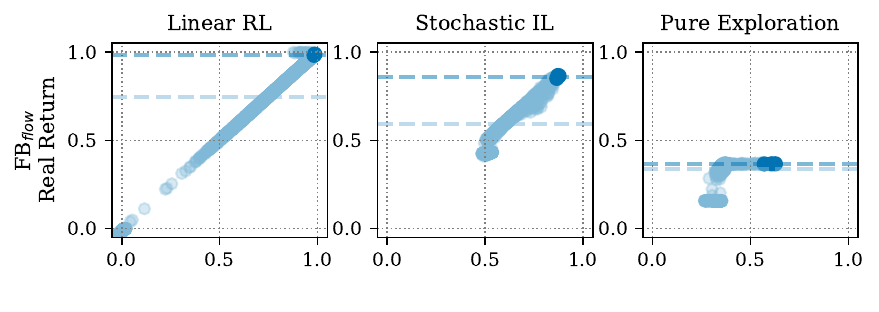}
    \end{minipage}%
    \begin{minipage}{0.5\linewidth}
    \includegraphics[width=\linewidth]{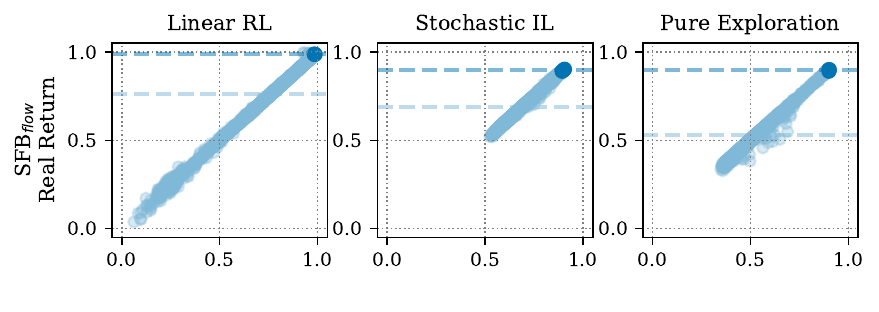}
    \end{minipage}%
    \vspace{-4mm}
    
    \begin{minipage}{0.5\linewidth}
    \includegraphics[width=\linewidth]{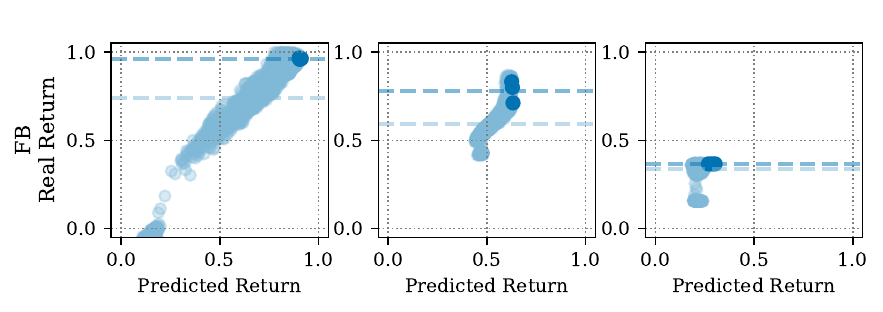}
    \end{minipage}%
    \begin{minipage}{0.5\linewidth}
    \includegraphics[width=\linewidth]{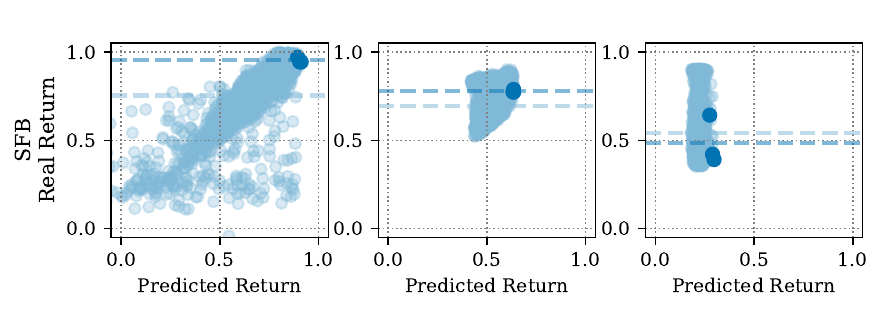}
    \end{minipage}%
    \vspace{-1mm}

    \caption{Quantitative results over several General RL objectives in a didactic environment. The $x$-axis and $y$-axis represent, respectively, offline performance estimates, and ground-truth performance in the environment. Each dot represents a policy sampled from each method across 3 seeds; for each seed, a darker dot marks the best policy according to offline evaluation. Horizontal lines represent the mean performance over points with the respective color. The policies captured by Soft FB (right) are more expressive, and the top policies affording to offline evaluation outperform, on average, those trained by FB (left). Explicit measure models (top) are more accurate.}
    \label{fig:quantitative}
    \vspace{-4mm}
\end{figure*}

\section{Experiments}
\label{sec:exp}

We now complement the formal analysis of \method with a detailed empirical evaluation.
We will first present qualitative results in an easily interpretable setting, and then evaluate how \method performs across zero-shot, general utility objectives, including imitation learning, pure exploration and constrained reinforcement learning.
Finally, we include an evaluation on standard, high dimensional deep reinforcement learning benchmarks \citep{yarats2022exorl}.

\begin{figure*}
    \centering
    \vspace{-2mm}
    \hspace{1.8cm} \texttt{walker} \hspace{2.4cm} \texttt{cheetah} \hspace{2cm} \texttt{quadruped} \hspace{2.4cm} \texttt{maze} \hspace{2.1cm}
    
    \includegraphics[width=1.0\linewidth]{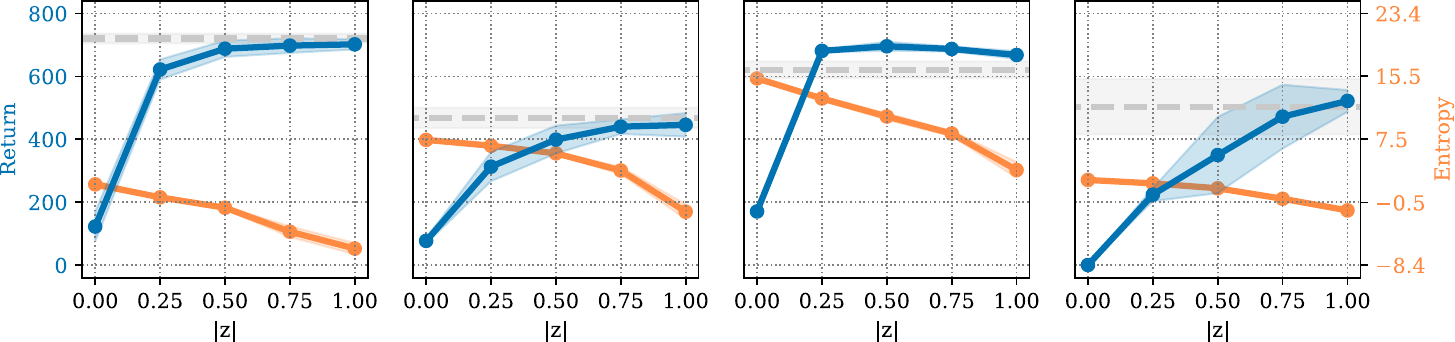}
    \vspace{-5mm}
    \caption{Zero-shot cumulative returns (in blue) and step-wise policy entropy (in orange) of \method for different levels of entropy regularization in DMC, averaged over linear tasks. As entropy regularization decreases, returns generally improve, eventually matching the performance of FB (in grey), or surprisingly exceeding it in \texttt{quadruped}. Shaded areas represent $95\%$ CIs over 5 seeds.\looseness -1}
    \label{fig:mujoco}
    \vspace{-5mm}
\end{figure*}

\subsection{Qualitative evaluation}
\label{subsec:qualitative}

Let us consider a simple, continuous environments, with 
\begin{align}
    & \mathcal{S} = \mathcal{A} = [-1, 1] \times [-1, 1], \; \mu_0 = \text{Dirac(0, 0)} \; \text{ and } \\
    & P(s' | s, a) = \begin{cases}
    \text{Dirac}(a) & \text{if } s=(0,0)\\
    \text{Dirac}(s) & \text{otherwise}.
\end{cases}
\end{align}
This didactic environment mimics a bandit-like MDP: the agent stays indefinitely in a state dictated by the action chosen in the initial state. We collect a dataset by executing actions uniformly at random, and train Soft FB over this data.
We start by defining a simple goal-reaching task as $R(s|g) = \mathbf{1}_{\|s-g\| < 0.2}$ and computing the corresponding reward embedding $z^\star=\frac{BR}{\|BR\|}$; for instance, we may fix $g=(0.0, 0.5)$.
We may then introduce a controlled amount of entropy regularization through a scaled reward embedding $z=\alpha z^\star$ with $\alpha \in [0, 1]$, and retrieve an optimal policy for all corresponding maximum entropy RL objectives in a zero-shot fashion.
If the norm of $z$ is set to $1$, the optimal deterministic policy should be retrieved, while very low norms should return near-uniform policies. \looseness -1

We observe this exact behavior in Figure \ref{fig:qualitative}. From left to right, we increase the norm of the task embedding, thus decreasing the degree of entropy regularization, and plot samples from $\pi_\psi(\cdot |s_0, z)$, as well as the entropy-\textit{un}regularized action-value function $Q_R(s_0, \cdot, z) = F_\theta(s_0, \cdot, z)^\top z$ as a function of 2D actions.
As expected, we observe that policies become more deterministic as the entropy regularization decreases, while still optimizing for the goal-reaching objective.
The state-action value function is highest around the goal, and scales linearly in the norm of $z$ by construction.
\begin{table}[ht]
\vspace{-1mm}
\centering
\caption{Zero-shot performance over several General Utilities in the didactic environment. For each inference technique, the best policy retrieved by SFB constitutes, on average, a better solution. Results are averaged across 3 seeds, and bold when their 95\% confidence intervals intersects with that of the best score.}
\label{tab:results}
\setlength{\tabcolsep}{2.5pt}
\begin{tabular}{l cccc}
\toprule
& FB & FB$_{flow}$ & SFB & SFB$_{flow}$ \\
\midrule
Linear RL & $0.96$ \tiny $\pm .01$ & $\textbf{0.99}$ \tiny $\pm .01$ & $0.95$ \tiny $\pm .01$ & $\textbf{0.99}$ \tiny $\pm .01$  \\
Goal-reaching RL & $\textbf{1.00}$ \tiny $\pm .01$ & $\textbf{1.00}$ \tiny $\pm .01$ & $\textbf{1.00}$ \tiny $\pm .01$ & $\textbf{1.00}$ \tiny $\pm .01$  \\
Deterministic IL & $0.78$ \tiny $\pm .06$ & $0.86$ \tiny $\pm .01$ & $0.78$ \tiny $\pm .01$ & $\textbf{0.90}$ \tiny $\pm .01$  \\
Stochastic IL & $0.67$ \tiny $\pm .01$ & $0.77$ \tiny $\pm .03$ & $0.79$ \tiny $\pm .02$ & $\textbf{0.83}$ \tiny $\pm .01$  \\
Pure Exploration & $0.37$ \tiny $\pm .01$ & $0.37$ \tiny $\pm .01$ & $0.49$ \tiny $\pm .13$ & $\textbf{0.90}$ \tiny $\pm .01$  \\
Constrained RL & $0.00$ \tiny $\pm .01$ & $0.00$ \tiny $\pm .01$ & $0.00$ \tiny $\pm .01$ & $\textbf{0.65}$ \tiny $\pm .52$  \\
Robust RL & $0.01$ \tiny $\pm .01$ & $0.79$ \tiny $\pm .09$ & $0.39$ \tiny $\pm .26$ & $\textbf{0.96}$ \tiny $\pm .02$  \\
\midrule
Average & 0.54 & 0.68 & 0.63 & \textbf{0.89}  \\
\bottomrule
\end{tabular}
\vspace{-5mm}
\label{tab:quantitative}
\end{table}

\subsection{Quantitative evaluation}

Within the environment outlined in the previous section, we now evaluate the performance of policies retrieved by Soft FB and FB, respectively.
We thus consider a set of different General RL objectives: (i) standard linear RL, (ii) goal-reaching RL, imitation of a (iii) deterministic or (iv) stochastic agent, (v) pure exploration, (vi) constrained RL and (vii) robust RL. Both IL objectives are formulated through minimization of forward KL divergences between successor measures; specifics of each objective can be found in Appendix \ref{app:implementation}.
For each objective and algorithm, we search for the optimal policy according to the inference procedures outlined in Section \ref{sec:inference}: we sample $1024$ task embeddings $z$, select the best one according to the measure model (implicit or explicit) and evaluate it in the environment.
Table \ref{tab:quantitative} compares the performance of Soft FB and FB, optionally relying on an explicit model for inference.
As expected, the performance gap between FB and Soft FB on objectives admitting a deterministic solution is minimal (the first three); however, when deterministic policies are not sufficient, Soft FB achieves better performance. Furthermore, we observe that relying on an explicit measure models for policy evaluation increases performance, confirming the effectiveness of flow modeling \citep{farebrother2025temporal} in stochastic settings.
This is further confirmed in Figure \ref{fig:quantitative}, which relates offline performance estimates to ground-truth performance in the environment in three representative tasks. These two are more strongly correlated\footnote{We report rank-based correlation coefficients and numerical results in Appendix \ref{app:correlation}.} when the former is estimated through an explicit measure model (SFB$_{flow}$, FB$_{flow}$); SFB$_{flow}$ performs the best as it also retrieves a richer class of policies (wider spread on the $y$-axis).

\subsection{High-dimensional evaluation}

So far, evaluation has largely focused on a continuous, yet low-dimensional setting, in which we demonstrated the expressiveness of policies retrieved by Soft FB.
We now extend our evaluation to standard, high-dimensional continuous control benchmarks, in order to demonstrate scalability.
We consider the Deepmind Control Suite (DMC, \citealp{tassa2018deepmind}), which includes tasks across four embodiments, and follow the established evaluation protocol in \citep{touati2023does} relying on exploratory data \citep{yarats2022exorl}.
We first consider classic linear objectives:
Figure \ref{fig:mujoco} tracks the return of policies trained with \method and their entropy when conditioned on scaled task embedding $z$ obtained by linear regression (i.e., the standard task inference procedure for Forward-Backward algorithms). As expected, decreasing the vector norm induces more stochastic behavior, which results in performance degradation. However, we notice that, for low entropy regularization, performance matches or exceeds that of policies trained by FB, confirming that the richness of policies does not come at the cost of performance. \looseness -1
\begin{table}[H]
\vspace{-1mm}
\caption{Zero-shot performance on general utilities in DMC. We consider entropy maximization and KL minimization with respect to a deterministic or stochastic expert (i.e. an optimal policy for a linear task, e.g. \texttt{walker-walk}, with injected Gaussian noise in the stochastic case). Scores are averages over 5 seeds, with 95\% confidence intervals and bold numbers signaling overlaps.}
\label{tab:mujoco}
\begin{center}
\begin{small}
\begin{tabular}{llcc}
\toprule
& & FB$_{flow}$ & Soft FB$_{flow}$ \\
\midrule
\multirow{6}{*}{\rotatebox[origin=c]{90}{\texttt{walker}}} 
& $\mathcal{H}(\mathrm{M}_\mathcal{S}^\pi)$ & $12.56$ \tiny $\pm .91$ & $\textbf{13.97}$ \tiny $\pm .25$ \\
& $\mathcal{H}(\mathrm{M}^\pi)$ & $9.86$ \tiny $\pm .10$ & $\textbf{18.01}$ \tiny $\pm .03$ \\
& $-KL(\mathrm{M}_\mathcal{S}^\pi; \mathrm{M}^{\pi_{stoch}}_\mathcal{S})$ & $\textbf{-5.35}$ \tiny $\pm .24$ & $\textbf{-5.25}$ \tiny $\pm .14$ \\
& $-KL(\mathrm{M}^\pi; \mathrm{M}^{\pi_{stoch}})$ & $-9.40$ \tiny $\pm .17$ & $\textbf{-1.48}$ \tiny $\pm .11$ \\
& $-KL(\mathrm{M}_\mathcal{S}^\pi; \mathrm{M}^{\pi_{det}}_\mathcal{S})$ & $\textbf{-5.69}$ \tiny $\pm .27$ & $\textbf{-5.69}$ \tiny $\pm .02$ \\
& $-KL(\mathrm{M}^\pi; \mathrm{M}^{\pi_{det}})$ & $-9.43$ \tiny $\pm .18$ & $\textbf{-1.53}$ \tiny $\pm .09$ \\
\midrule
\multirow{6}{*}{\rotatebox[origin=c]{90}{\texttt{cheetah}}} 
& $\mathcal{H}(\mathrm{M}_\mathcal{S}^\pi)$ & $11.74$ \tiny $\pm .43$ & $\textbf{13.63}$ \tiny $\pm .10$ \\
& $\mathcal{H}(\mathrm{M}^\pi)$ & $11.68$ \tiny $\pm .89$ & $\textbf{17.83}$ \tiny $\pm .07$ \\
& $-KL(\mathrm{M}_\mathcal{S}^\pi; \mathrm{M}^{\pi_{stoch}}_\mathcal{S})$ & $-5.56$ \tiny $\pm .31$ & $\textbf{-4.56}$ \tiny $\pm .31$ \\
& $-KL(\mathrm{M}^\pi; \mathrm{M}^{\pi_{stoch}})$ & $-7.25$ \tiny $\pm .76$ & $\textbf{-1.02}$ \tiny $\pm .08$ \\
& $-KL(\mathrm{M}_\mathcal{S}^\pi; \mathrm{M}^{\pi_{det}}_\mathcal{S})$ & $-5.45$ \tiny $\pm .65$ & $\textbf{-4.38}$ \tiny $\pm .18$ \\
& $-KL(\mathrm{M}^\pi; \mathrm{M}^{\pi_{det}})$ & $-7.55$ \tiny $\pm .79$ & $\textbf{-1.23}$ \tiny $\pm .08$ \\
\midrule
\multirow{6}{*}{\rotatebox[origin=c]{90}{\texttt{quadruped}}} 
& $\mathcal{H}(\mathrm{M}_\mathcal{S}^\pi)$ & $13.37$ \tiny $\pm .43$ & $\textbf{14.43}$ \tiny $\pm .24$ \\
& $\mathcal{H}(\mathrm{M}^\pi)$ & $10.11$ \tiny $\pm .08$ & $\textbf{19.10}$ \tiny $\pm .01$ \\
& $-KL(\mathrm{M}_\mathcal{S}^\pi; \mathrm{M}^{\pi_{stoch}}_\mathcal{S})$ & $-5.44$ \tiny $\pm .28$ & $\textbf{-4.64}$ \tiny $\pm .20$ \\
& $-KL(\mathrm{M}^\pi; \mathrm{M}^{\pi_{stoch}})$ & $-9.92$ \tiny $\pm .16$ & $\textbf{-0.76}$ \tiny $\pm .04$ \\
& $-KL(\mathrm{M}_\mathcal{S}^\pi; \mathrm{M}^{\pi_{det}}_\mathcal{S})$ & $-5.61$ \tiny $\pm .30$ & $\textbf{-4.52}$ \tiny $\pm .16$ \\
& $-KL(\mathrm{M}^\pi; \mathrm{M}^{\pi_{det}})$ & $-9.94$ \tiny $\pm .20$ & $\textbf{-0.79}$ \tiny $\pm .02$ \\
\midrule
\multirow{6}{*}{\rotatebox[origin=c]{90}{\texttt{maze}}} 
& $\mathcal{H}(\mathrm{M}_\mathcal{S}^\pi)$ & $\textbf{10.83}$ \tiny $\pm .56$ & $\textbf{11.22}$ \tiny $\pm .70$ \\
& $\mathcal{H}(\mathrm{M}^\pi)$ & $10.52$ \tiny $\pm .66$ & $\textbf{15.54}$ \tiny $\pm .16$ \\
& $-KL(\mathrm{M}_\mathcal{S}^\pi; \mathrm{M}^{\pi_{stoch}}_\mathcal{S})$ & $\textbf{-6.49}$ \tiny $\pm .36$ & $\textbf{-5.25}$ \tiny $\pm .90$ \\
& $-KL(\mathrm{M}^\pi; \mathrm{M}^{\pi_{stoch}})$ & $-6.45$ \tiny $\pm .61$ & $\textbf{-1.39}$ \tiny $\pm .31$ \\
& $-KL(\mathrm{M}_\mathcal{S}^\pi; \mathrm{M}^{\pi_{det}}_\mathcal{S})$ & $-6.57$ \tiny $\pm .61$ & $\textbf{-5.07}$ \tiny $\pm .69$ \\
& $-KL(\mathrm{M}^\pi; \mathrm{M}^{\pi_{det}})$ & $-7.27$ \tiny $\pm .58$ & $\textbf{-2.51}$ \tiny $\pm .35$ \\
\bottomrule
\end{tabular}
\end{small}
\end{center}
\vspace{-5mm}
\end{table}

As standard objectives in DMC are exclusively linear, we additionally define several general utilities: pure exploration (i.e. entropy maximization) and imitation of a deterministic or stochastic expert (i.e. KL minimization) over state and state-action measures.
Appendix \ref{app:implicit} reports an in-depth description of these objectives, and a nuanced discussion of results, including additional baselines.
For brevity, we also compare the performance of Soft FB and FB with explicit measure models in Table \ref{tab:mujoco}.
While performance is explainably similar for linear tasks, we find that a significant gap appears as soon as a deterministic optimal policy does not exist, confirming the effectiveness of Soft FB in producing a richer class of policies, with strong performance beyond linear rewards.

\section{Related Works}
\label{sec:related}

\textbf{Zero-shot Reinforcement Learning} \;
In its purest instantiation, reinforcement learning is centered on optimizing a single, scalar reward function \citep{sutton1998reinforcement}. Despite its flexibility \citep{silver2021reward}, this formulation does not adapt to changes in the objective without re-training.
One possible solution to this issue has risen to prominence in deep reinforcement learning, and revolves around goal-conditioned policies \citep{andrychowicz2017hidsight, eysenbach2022contrastive} and universal value functions approximators \citep{schaul2015universal}. Such methods normally specialize to specific, parameterized reward classes (e.g., Dirac \citep{tian2021model}), or combinations thereof \citep{frans2024unsupervised}, and can train conditional policies through self-supervision \citep{ghosh2023reinforcement}, with impressive empirical results \citep{akkaya2019solving}.

A parallel line of work is instead aimed at solving more general tasks, specifically by learning policy representations that enable efficient evaluations for any reward function.
Fundamental works in this direction revolve around the design and estimation of successor features and representations \citep{dayan1993improving, barreto2017successor}; interestingly, in this context, \citet{hunt2019composing} formulate a related entropy-regularized variant, but remain focused on linear rewards.

Forward-backward methods \citep{blier2021learning, touati2021learning, touati2023does}  leverage a low-rank decomposition of occupancy measures to extract task embeddings in a successor feature framework.
These algorithms' flexibility allows direct application to imitation learning \citep{pirotta2024fast}, self-supervised exploration \citep{urpi2025epistemically}, while also achieving competitive performance on high-dimensional simulated environments \citep{tirinzoni2024zero}. Recent works have further explored alternative parameterizations \citep{cetin2024finer, bagatella2025td}, more capable generative models for occupancies \citep{farebrother2025temporal}, fast online exploration or adaptation \citep{sikchi2025fast, urpi2025epistemically, rupf2025optimistic}.
However, forward-backward algorithms remain constrained to linear RL problems and, practically, deterministic policies, which becomes a limitation on more complex objectives, involving multimodal data distributions or exploration objectives.
Our work directly addresses this limitation.

\textbf{Non-linear Reinforcement Learning}\;
Much of the existing RL machinery builds upon the assumption that the objective may be broken down in additive terms, each of whom may be traced back to a single state-action pair \citep{sutton1998reinforcement}. This is referred to as RL with additive rewards, or Linear RL.
Convex RL \citep{zahavy2021reward, geist2021concave, mutti2022challenging} encompasses a richer class of objective, such as pure exploration \citep{hazan2019provably}, active learning in MDPs \cite{mutny2023active} and distribution matching \citep{kostrikov2020imitation, rupf2024zero}. A further generalization leads to RL with General Utilities \citep{zhang2020variational}, for which we provide more related works in Appendix \ref{app:related} due to space constraints.
Solutions to Convex or General Utilities may be found by iterative procedures \citep{geist2021concave}, which involve a low-level MDP-solving routine, or adversarial optimization schemes \citep{zahavy2021reward}. This generally produces mixture, non-Markovian policies, while \method returns a single, Markov policy.
Zero-shot methods have been applied to non-linear problems before \citep{pirotta2024fast}, but are restricted to imitation learning and only take non-additivity into account during inference.
To the best of our knowledge, our work is the first in exploring zero-shot solutions to arbitrary General RL problems in a principled and scalable way. \looseness -1

\section{Conclusion}
\label{sec:conclusion}

At its core, this work proposes a novel extension of zero-shot reinforcement learning beyond linear rewards.
We introduce a soft forward-backward algorithm, which leverages a simple entropy regularization mechanism to capture stochastic behaviors in a dynamical system.
At inference, the space of behaviors can be searched efficiently to retrieve an approximately optimal policy for an arbitrary general utility. \looseness -1

\textbf{Limitations and Future Work}\; 
While Soft FB may provably retrieve all stochastic policies, this requires infinite-dimensional task representations and expressive actors: in practical settings, a narrower set of policies will be retrieved. Formally studying the properties of this set represents an important direction for future work.
Accurate search over policies relies on precise modeling of successor measures: while a dedicated generative model suffices, it also increases the computational costs during training.
Furthermore, while we found simple zero-order optimization to be sufficient for search among policies, more involved optimization techniques may further scale this approach to higher-dimensional representation spaces. 
A further interesting extension would generalize our framework to arbitrary successor-feature-based algorithms.

Soft FB introduces a first-of-its-kind extension of zero-shot RL beyond linear rewards. We hope that this work represents a step forward in bringing fundamental works on general utilities closer to application in practical settings.

\newpage

\section*{Acknowledgements}
We would like to thank Andrea Tirinzoni, Núria Armengol Urpí, Marin Vlastelica, Pavel Kolev, Yifan Hu and Ehsan Sharifian for the fruitful discussions and valuable feedback.
Marco Bagatella is supported by the Max Planck ETH Center for Learning Systems. This project was supported in part by the Swiss National Science Foundation under NCCR Automation, grant agreement 51NF40 180545.

\section*{Impact Statement}

This paper presents an algorithm for zero-shot optimization of general objectives in MDPs. As such, its potential impact falls within the societal consequences of machine learning and robotics at large.

\bibliography{main}
\bibliographystyle{icml2026}

\newpage
\appendix
\onecolumn

\section{Theoretical results and proofs}
\label{app:proofs}

This section starts by reporting proofs for the main formal results. We later present some additional insights in smoothness and interpolation properties that arise with Soft FB's entropy regularization.

\subsection{Proofs}

This section contains proofs for Theorems \ref{th:maxent} and \ref{th:generalutility}; the latter is preceded by a useful intermediate Lemma.

\maxent*
\begin{proof}
    This result can be proven through a direct generalization of the results in \citet{haarnoja2018soft} and \citet{touati2021learning}.
    Starting from Equation \ref{eq:soft_policy}, setting $z=BR$, we can observe that
    \begin{align}
        \pi_z & \propto \exp (F_z^\top z + M^z\mathcal{H}_{\pi_z}) \\
        & \stackrel{z=BR}{=}\exp (F_z^\top BR + M^z\mathcal{H}_{\pi_z}) \\
        & \stackrel{Eq. \ref{eq:low_rank_decomposition}}{=} \exp\big(M^z(R+\mathcal{H}_{\pi_z})\big) \\
        & = \exp(Q^z_\text{soft}).
    \end{align}
    Up to normalization constants, the policy satisfies the optimality criterion from Theorem 1 in \citet{haarnoja2017reinforcement}, and is thus the optimal maximum entropy policy, i.e. $\pi_z = \mathop{\text{argmax}}_\pi M^\pi(R + \mathcal{H}_{\pi_z})$.
\end{proof}

\begin{lemma}
    Let $\tilde \Pi_z$ be the set of policies retrieved by Soft FB and $\pi \in \Pi$ be an arbitrary Markov policy. If $\pi(a|s) > 0$ for every $(s,a) \in \mathcal{S \times A}$ (i.e., $\pi$ has complete support), then $\pi \in \tilde \Pi_z$.
    \label{lemma:support}
\end{lemma}
\begin{proof}
    We will show that a maximum entropy problem admitting $\pi$ as its optimal policy can be constructed, which will imply that $\pi$ is part of the set of solutions to entropy-regularized instances retrieved by Soft FB by Theorem \ref{th:maxent}. Let $R(s,a)=\log \pi(a|s)$: since $\pi$ has complete support, $R(s,a)$ is bounded. Let us now consider its corresponding maximum entropy objective:
    \begin{align}        
    J^{\pi'}_\mathcal{H}(R) & = \mathop{\mathbb{E}}_{\mu_0, P, \pi'}\sum_{t=0}^{\infty}\gamma^t \Big(R(s_t, a_t) + \mathcal{H}\big(\pi'(\cdot|s)\big)\Big) \\
    & = \mathop{\mathbb{E}}_{\mu_0, P, \pi'}\sum_{t=0}^{\infty}\gamma^t \mathop{\mathbb{E}}_{a_t \sim \pi'(\cdot|s_t)}\Big(\log \pi(a_t|s_t) - \log \pi'(a_t|s_t)\Big) \\
    & = \mathop{\mathbb{E}}_{\mu_0, P, \pi'}\sum_{t=0}^{\infty}\gamma^t \Big( - D_\text{KL}\big(\pi'(\cdot|s_t)||\pi(\cdot|s_t\big) \Big) \leq 0.
    \end{align}
    This objective is an expected, discounted sum of negative KL distances, and therefore non-positive.
    As $D_\text{KL}(p||p)=0$, $J^{\pi}_\mathcal{H}(R^\star)=0$, and $\pi$ is the optimal policy for the maximum entropy objective with reward $R$.
    Theorem \ref{th:maxent} guarantees that there exist a reward embedding $z \in \mathcal{Z}$ recovering a solution $\pi_z$ for each maximum entropy problem; therefore, there exist a reward embedding $z \in \mathcal{Z}$ such that $\pi_z = \mathop{\text{argmax}}_{\pi' \in \Pi}J^{\pi'}_\mathcal{H}(R) = \pi$, and $\pi \in \tilde \Pi_z$.
\end{proof}

\generalutility*
\begin{proof}

    This proof will rely on the construction of a sequence of policies. We will first show that each policy is retrieved by Soft FB, and then that this sequence gets arbitrarily close to the optimum.

    Let us consider an arbitrary differentiable objective $f$ and its optimal policy $\pi^\star = \mathop{\text{argmax}}_{\pi \in \Pi} f(\mathrm{M}^\pi)$. Let $\bar \pi = \mathcal{U}(\mathcal{A})$ denote the uniform policy (i.e., the policy choosing each action with equal probability). We can construct a sequence of increasingly stochastic policies as a linear interpolation of the two through a parameter $\alpha \in (0, 1]$:
    \begin{equation}
        \pi_\alpha(a|s) := (1-\alpha)\pi^\star(a|s) + \alpha\bar\pi(a|s).
    \end{equation}
    For all $\alpha > 0$ and $(s,a) \in \mathcal{S \times A}$, $\pi_\alpha (a|s) > \alpha \bar\pi(a|s) > 0$. Thus, each policy $\pi_\alpha$ has full support. Through Lemma \ref{lemma:support} we can conclude that $\pi_\alpha \in \tilde \Pi_z$, that is the set of policies retrieved by Soft FB. We will now show that, as $\alpha \to 0$, $\pi_\alpha$ is approximately close to the optimum for $f$.

    We can first bound the distance in action space to the optimal policy $\pi^\star$ for each $s \in \mathcal{S}$, that is
    \begin{align}
        \|\pi^\star(\cdot|s) - \pi_\alpha(\cdot|s)\|_1 & = \sum_{a \ \in \mathcal{A}} | \pi^\star(a|s) - (1-\alpha) \pi^\star(a|s) - \alpha \bar \pi (a|s) | \\
        & = \sum_{a \ \in \mathcal{A}} \alpha | \pi^\star(a|s) - \bar \pi (a|s) | \\
        & \leq 2\alpha,
    \end{align}
    where the last inequality follows from the fact that $\pi^\star$ and $\bar\pi$ are probability distributions, and thus bounded in $[0,1]$.
    We can now apply the Simulation Lemma \citep{kearns2002near} to bound the distance between marginalized successor measures of the two policies:
    \begin{equation}
        \|\mathrm{M}^{\pi^\star} - \mathrm{M}^{\pi_\alpha}\|_1 \leq \frac{1}{1-\gamma} \max_{s \in \mathcal{S}}\|\pi^\star(\cdot|s) - \pi_\alpha(\cdot|s)\|_1 \leq \frac{2\alpha}{1-\gamma}
        \label{eq:bound_m}
    \end{equation}
    The last step requires bounding the difference $f(\mathrm{M}^\pi) - f(\mathrm{M}^{\pi_\alpha})$, which is possible through uniform continuity.
    Let us consider the set of all valid occupancy measures $\mathcal{K} = \{M^\pi\}_{\pi \in \Pi} \subset \mathbb{R}^{\mathcal{|S||A|}}$.
    Each element of $\mathcal{K}$ must respect the following linear constraints: $M(s,a) \geq 0$ for all $(s,a) \in \mathcal{S \times A}$ (i.e., probabilities must be non-negative), (ii) $\sum_{s,a \in \mathcal{S \times A}}M(s,a) = 1$ (i.e., $\mathrm{M}$ is a valid probability distribution), and (iii) $\sum_a \mathrm{M}(s,a) = (1-\gamma) \mu_0(s) + \gamma \sum_{s', a'}P(s|s', a')\mathrm{M}(s',a')$ (i.e., Bellman flow constraints respecting the dynamics of the MDP). $\mathcal{K}$ is thus closed and bounded, and by consequence a compact set.
    Furthermore, differentiability of $f$ implies continuity of $f$: by the Heine Cantor Theorem, a continuous function on a compact set is uniformly continuous.
    This is equivalent to the existence of a continuous, monotonic function $\omega_f:[0, \infty) \to [0, \infty)$ with $\lim_{\delta \to 0}\omega_f(\delta) = 0$ (i.e., a modulus of continuity) such that
    \begin{equation}
        |f(\mathrm{M})-f(\mathrm{M'})| \leq \omega_f(\|\mathrm{M} - \mathrm{M'}\|_1)
        \label{eq:uniform_cont}
    \end{equation}
    Combining the uniform continuity property in Equation \ref{eq:uniform_cont} with the bound on measures in Equation \ref{eq:bound_m}, we finally have
    \begin{equation}
        |f(\mathrm{M^{\pi^\star}})-f(\mathrm{M}^{\pi_\alpha})| \leq \omega_f(\|\mathrm{M}^{\pi^\star} - \mathrm{M}^{\pi_\alpha}\|_1) \leq \omega_f \bigg(\frac{2\alpha}{1-\gamma} \bigg)
    \end{equation}
    Since $\omega_f(x) \to 0$ as $x \to 0$, for any precision $\epsilon$, one can find a sufficiently small $\alpha$ such that $\omega_f(2\alpha(1-\gamma)^{-1}) < \epsilon$, and thus $\pi_\alpha$ achieves an arbitrarily close objective value to the optimum of $f$.
    As $\pi_\alpha \in \tilde \Pi_z$, we can conclude that there exist a policy among those retrieved by Soft FB that optimizes $f$ arbitrarily well.
\end{proof}

\subsection{Extended counterexample}
\label{app:counterexample}

\begin{wrapfigure}[15]{r}{0.4\textwidth}
  \begin{center}
    \includegraphics[width=0.4\textwidth]{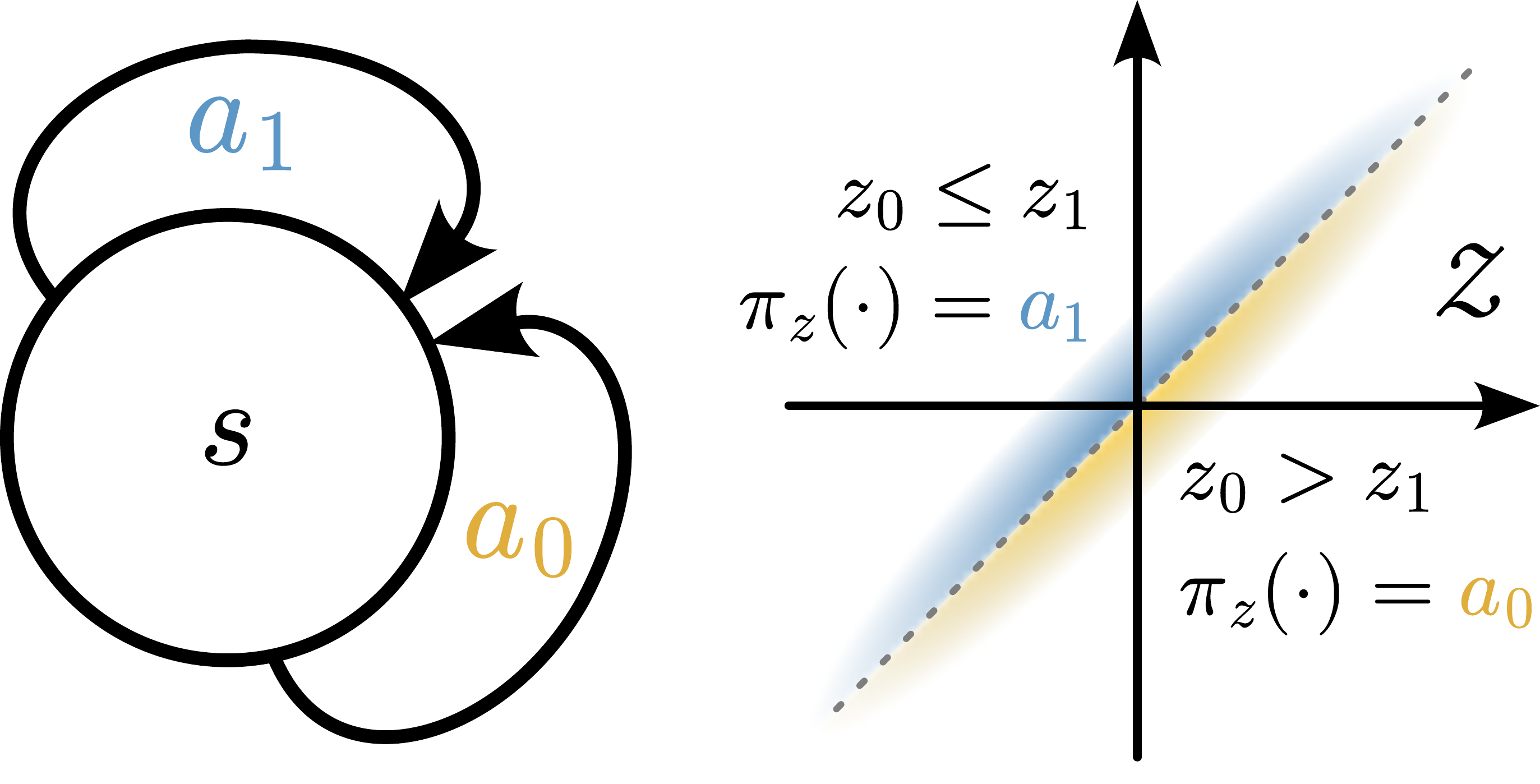}
  \end{center}
  \caption{The simple MDP referred to for the counterexample (left); we consider a 2-dimensional representation space $\mathcal{Z}$ (right), and policies $\pi_z$ that output the first or second action depending on the semi-plane their parameter $z$ belongs to (right).}
  \label{fig:counterexample}
\end{wrapfigure}

Section \ref{sec:motivation} anticipates a counterexample to clarify why policies retrieves by FB may fail to optimize a given GU. We will now formalize this counterexample in detail, while providing a visual description in Figure \ref{fig:counterexample}.

Let us consider an MDP $\mathcal{M}$ with a single state $\mathcal{S}=\{s_0\}$ and two actions $\mathcal{A}=\{a_0, a_1\}$. Let us also consider representations of size $d=2$ (i.e., $\mathcal{Z} = \mathbb{R}^2$).
Let us now set the backward representation matrix as the identity matrix (i.e., $B=I$).
By doing so, we can take (the transposes of) policy-conditional successor measure matrices $M_z$ as forward representation matrices (i.e., $F_z^\top = M_z$), and trivially satisfy Equation \ref{eq:low_rank_decomposition} as $F_z^\top B = M_z I = M_z$ for any policy $\pi_z$.
Furthermore, setting $B$ to the identity matrix projects all rewards functions $R \in \mathbb{R}^2$ to identical reward embeddings $z=BR=R$.

It remains to show that there is a family of policies that satisfies Equation \ref{eq:greedy_policy} as well.
We thus choose the family of $z$-parameterized policies $\Pi_z = \{\pi_z\}_{z \in \mathcal{Z}}$, where the policy parameterized by $z$ picks the action indexed by the largest of the two elements of $z$: $\pi_z(\cdot) = a_0$ if $z_0 > z_1$, else $a_1$. Intuitively, as $z=R$, these policies are simply choosing the action associated with the largest reward.
In order to check whether these policies satisfy Equation \ref{eq:greedy_policy}, we need to calculate its right hand-side; as $F_z^\top = M_z$, it suffices to compute the successor measure for all policies.
Let us start by considering $z_0 > z_1$: in this case $\pi_z(\cdot) = a_0$. When starting from $(s, a_0)$, the other state-action pair is never visited: $M_z(s, a_0, s, a_0) = 1$. When starting from $(s, a_1)$, this state-action pair is only visited at the beginning of the trajectory, and never again: $M_z(s, a_1, s, a_1) = (1-\gamma)\gamma := C$ with $0 \leq C < 1$. By considering that $M_z$ is a stochastic matrix, and similarly working through the case in which $z_0 \leq z_1$, we have
\begin{equation}
    F_z^\top = M_z =
    \begin{bmatrix}
    1 & 0 \\
    1-C & C 
    \end{bmatrix}
    \; \text{ if } z_0 > z_1 \text{, else }
    \begin{bmatrix}
    C & 1 - C \\
    0 & 1 \\
    \end{bmatrix}.
\end{equation}
We can now verify directly that Equation \ref{eq:greedy_policy} holds.
In the case in which $z_0 > z_1$, we have 
\begin{equation}
    F_z^\top z = \begin{bmatrix}
    1 & 0 \\
    1 - C & C 
    \end{bmatrix} \begin{bmatrix}
    z_0 \\
    z_1 
    \end{bmatrix} = \begin{bmatrix}
    z_0 \\
    (1-C)z_0 + Cz_1 
    \end{bmatrix}.
\end{equation}
Since $(1-C)z_0 + Cz_1 \stackrel {z_1<z_0}{\leq} (1-C)z_0 + Cz_0=z_0$, then the policy optimizes its own Q-function: $\pi_z(\cdot) = a_0 =\mathop{\text{argmax}}_\mathcal{A} F^\top_z z$.
Similarly, in the case in which $z_0 \leq z_1$, we have 
\begin{equation}
    F_z^\top z = \begin{bmatrix}
    C & 1-C \\
    0 & 1 
    \end{bmatrix} \begin{bmatrix}
    z_0 \\
    z_1 
    \end{bmatrix} = \begin{bmatrix}
    Cz_0 + (1-C)z_1 \\
    z_1 
    \end{bmatrix}.
\end{equation}
This time, $Cz_0 + (1-C)z_1 \stackrel{z_1 \geq z_0}{\geq} C(z_1) + (1-C)z_1 = z_1$, and the policy still optimizes its own Q-function: $\pi_z(\cdot) = a_1 = \mathop{\text{argmax}}_\mathcal{A}F_z^\top z$. Since these two cases ($z_0 > z_1$ and $z_0 \leq z_1$) cover all policies $\pi_z \in \Pi_z$, Equation \ref{eq:greedy_policy} holds.

All policies $\pi_z$ are strictly deterministic. We can now consider a convex RL objective such as \textit{pure exploration} over states and actions, i.e. $J^\pi = f(\mathrm{M}^\pi) = \mathcal{H}(\mathrm{M}^\pi)$, where $\mathcal{H}(\cdot)$ denotes Shannon entropy over state-action pairs. For this objective, all policies in $\Pi_z$ are actually \textit{minimizers}: $J^{\pi_z} = 0$ for all $z \in \mathcal{Z}$. The optimal policy is uniform over the two actions (i.e., $\pi^\star(\cdot) = \mathcal{U}(\mathcal{A})$), achieves $J^{\pi^\star}=\log 2$ and does not belong to $\Pi_z$. This thus represents an instance in which the set of policies retrieved by FB may not include the optimal policy for some GU.

\subsection{Additional remarks}
One well-known property of entropy regularization in RL is that of inducing smoothness in the mapping from rewards to optimal behaviors \citep{geist2019theory, husain2021regularized}. This property is reflected in Soft FB, as we discuss in this section.

\begin{remark}
    Let us assume that Equations \ref{eq:low_rank_decomposition} and \ref{eq:soft_policy} hold, and that backward representations $B$ are full-rank. Let us now consider the map $g: \mathcal{Z} \to \mathcal{M}$ from the reward representation space $\mathcal{Z}$ to the space of feasible occupancy measures $\mathcal{M}$, such that $g(z) = M^{\pi_z}$. The map $g$ is $\mathcal{C}^\infty$ (smooth).
\end{remark}
\begin{proof}
    The map $g$ is a composition of several functions: we will show that each of these functions is, in turn, smooth.
    \begin{itemize}
    \item As established in Theorem \ref{th:maxent}, each policy $\pi_z \in \tilde \Pi_z$ is the regularized optimal policy for reward $R$ if $z=BR$.
    We can thus project the reward representation $z$ to its reward as $R=(B^\top B)^{-1}B^\top z$, where $B^\top B$ is invertible as $B$ is full rank.
    This map $z \mapsto R$ is linear and thus also smooth in $z$.
    \item Let us define the optimal entropy-regularized state-action value function for reward $R$: $Q^\star = \max_{\pi \in \Pi} M^\pi (R + \mathcal{H}_\pi)$.
    This is the unique fixed point of the Soft Bellmann optimality operator $\mathcal{T}_r$. We may then define its root finding function as $G(Q, z) = Q - \mathcal{T}_r(Q) = 0$.
    Component-wise for each state-action pair $(s,a)$ it takes the form
    \begin{align}
        G_{(s,a)}(Q, r) = Q(s,a) - \left(
            r(s,a) + \gamma \sum_{s'} P(s'|s,a) \log \sum_{a'} \exp Q(s',a')
        \right)
    \end{align}
    which, as a composition of linear function and the LogSumExp function, is itself smooth.
    In order to apply the Implicit Function Theorem, we examine the Jacobian of $G$ w.r.t.\ $Q$. 
    The partial derivative takes the form
    \begin{align}
        \frac{\partial G_{(s,a)}}{\partial Q_{(s',a')}} 
        = \delta_{(s,a) = (s',a')} - \gamma P(s'|s,a) \pi_\text{soft}(a'|s') \quad\text{ where }\quad \pi_\text{soft}(a'|s') = \frac{\exp Q(s',a')}{\sum_b \exp Q(s',b)}.
    \end{align}
    In matrix form, this is $J_Q = \mathrm{I} - \gamma P^{\pi_\text{soft}}$,
    which has eigenvalues with norm of at most $\gamma < 1$, making it non-singular.
    We can thus apply the Implicit Function Theorem, and confirm that the map $R \mapsto Q^\star$ is unique and smooth.
    \item By Theorem \ref{th:maxent} $\pi_z = \text{softmax}(Q^\star)$: this map is smooth in $Q^\star$ because of the smoothness of the softmax operator.
    \item Finally we have $M^z = (\mathrm{I} - \gamma P^{\pi_z})^{-1}$. First, we notice that $\mathrm{I} - \gamma P^{\pi_z}$ is smooth as a composition of smooth functions, and second, that it is non-singular, as we showed for $J_Q$ above.
    As such, the matrix inverse is also smooth, making $\pi_z \mapsto M^z$ a smooth function.
    \end{itemize}
    As $g$ is a composition of smooth functions ($z \mapsto R \mapsto Q^\star \mapsto \pi_z \mapsto M^z$), we can conclude that it is also smooth.
\end{proof}

This property allows smooth interpolation of behaviors; however we note that the solutions of linearly interpolated task vectors do not necessarily lie on a line in the space of successor measures. In contrast, the remark above does not hold for FB in general: repurposing the counterexample in Section \ref{app:counterexample}, and in particular Figure \ref{fig:counterexample}, we can see that changes in $z$ results in a sharp change in $\pi_z$ and $M^z$ when $z_1=z_2$.

\section{Further Related Work}
\label{app:related}

\textbf{Non-linear Reinforcement Learning}\;
RL with General Utilities \citep{zhang2020variational} may not be directly solved through standard dynamic programming or actor-critic approaches, as the key assumption behind TD learning is violated. On the other hand, policy gradient theorems may still be derived \citep{zhang2020variational, zhang2021convergence, kumar2022policy, barakat2023reinforcement}.
Specific algorithms for subclasses of convex RL problems (such as maximum entropy RL \citep{haarnoja2018soft} or imitation learning \citep{kostrikov2020imitation}) can be derived, for instance by building upon duality \citep{sikchi2023imitation}, but do not generalize to arbitrary objectives.
Recent works have also focused on non-linear objectives displaying set properties \citep{de2024global}, while a parallel line of work has departed from optimizing for expected occupancy to consider finite trial objectives \citep{mutti2023convex, jain2023maximum}.
While theoretical analysis of algorithms for non-linear RL has received significant attention \citep{huang2024occupancy, barakat2025global}, their scalability to high-dimensional problem has not been extensively tested. Similarly, most algorithms deal with the standard online setting, and solve for one objective at a time.

\section{Explicit Measure Models}
\label{app:tdflow}

The inference procedure of Soft FB requires accurate measure estimates. As we mention in Section \ref{sec:inference}, while a rough estimate can be extracted from an implicit measure model (i.e., relying on forward and backward representations directly for importance sampling), we find that training an explicit measure model results in much more reliable estimates.
Following recent work on geometric horizon models \citep{farebrother2025temporal, zheng2025intention}, we train a generative model of $z$-conditional successor measures $M^z$ through a temporal-difference, conditional flow-matching objective; we briefly report the core idea as follows.

Let us consider a time-dependent probability path $m_t: \mathcal{S \times A \to \mathcal{S}}$ for $t \in [0, 1]$ describing a process transporting samples from an initial distribution $m_0=$ to $m_1 = M^z$. This process is described by a vector field $v_t: \mathcal{S \times S \times A \to S}$, which may be integrated to produce the flow $\psi_t: \mathcal{S \times S \times A \to S}$ governing the probability path:

\begin{gather*}
    \frac{\mathrm{d}}{\mathrm{d}t} \psi_t(x \mid s, a) = v_t(\psi_t(x \mid s, a) \mid s, a), \quad \psi_0(x \mid s, a) = x \iff \psi_t(x \mid s, a) = x + \int_0^t v_\tau(\psi_\tau(x \mid s, a) \mid s, a) \, \mathrm{d}\tau.
\end{gather*}

If $X_t := \psi_t(X_0 \mid S, A) \sim m_t(\cdot \mid S, A)$ for $X_0 \sim m_0$, $v_t$ is said to generate $m_t$; estimation of $v_t$ is facilitated by introducing a conditioning (e.g., $Z = X_1$) and choosing a Gaussian conditional probability path $p_{t,Z}:=\mathcal{N}(\cdot|tX_1, (1-t^2)I)$, which grants a closed form for a \textit{conditional} vector field $u_{t,X_1}(x):=(X_1 - x)/(1-t)$. Furthermore, off-policy learning is possible by bootstrapping (i.e., using the estimated vector field to sample from measures conditional on the next state), and variance reduction can be carried out by directly matching the the estimated vector field.
This results in the TD$^2$-CFM objective for training a parameterized estimator of the vector field $\tilde v_{t, \theta}: \mathcal{S \times S \times A \to S}$, which we report:
\begin{gather*}
    \vec{\mathcal{L}}(\theta) = \mathbb{E}_{\rho, t, Z, \vec{X}_t} \left[ \left\| \tilde{v}_t(\vec{X}_t \mid S, A; \theta) - \vec{u}_{t|Z}(\vec{X}_t \mid Z) \right\|^2 \right], \\
    \text{where } Z = X_1 \sim P(\cdot \mid S, A),~ \vec{X}_t \sim p_{t|Z}(\cdot \mid Z), \\
    \widetilde{\mathcal{L}}(\theta) = \mathbb{E}_{\rho, t, \widetilde{X}_t} \left[ \left\| \tilde{v}_t(\widetilde{X}_t \mid S, A; \theta) - \bar{\tilde{v}}_t(\widetilde{X}_t \mid S', \pi_z(S')) \right\|^2 \right], \\
    \text{where } X_0 \sim p_0,~ S' \sim P(\cdot \mid S, A),~ \widetilde{X}_t = \widetilde{\psi}_t(X_0 \mid S', \pi_z(S')), \\
    \mathcal{L}_{\text{TD}^2\text{-CFM}}(\theta) = (1 - \gamma)\vec{\mathcal{L}}(\theta) + \gamma{\widetilde{\mathcal{L}}}(\theta),
\end{gather*}
where $\rho$ represents the empirical distribution of states in the same dataset used for training Soft FB. We refer to \citet{farebrother2025temporal} for a detailed discussion, and for implementation details, which we closely follow.

\section{Additional Experimental Results}

\subsection{Correlation}
\label{app:correlation}

The policy search procedure outlined in Section \ref{sec:inference} relies on offline performance estimates, which are only effective if they correlate with ground truth performance in the environment.
We find that the degree of correlation depends on the objective and on the technique used for sampling from the measure: explicit models (denoted with, e.g., SFB$_{flow}$) generally result in more reliable policy evaluation.
For completeness, we compute Pearson correlation coefficients across the results from Table \ref{tab:quantitative} in Table \ref{tab:correlation}.
\begin{table}[ht]
\centering
\caption{Spearman's rank correlation coefficients between offline policy estimates and ground-truth performance in the didactic environment.}
\label{tab:correlation}
\begin{tabular}{l cccc}
\toprule
& FB & FB$_{flow}$ & SFB & SFB$_{flow}$ \\
\midrule
Linear RL & $0.98$ \tiny $\pm .01$ & $\textbf{1.00}$ \tiny $\pm .01$ & $0.83$ \tiny $\pm .02$ & $\textbf{1.00}$ \tiny $\pm .01$  \\
Goal-reaching RL & $0.90$ \tiny $\pm .01$ & $\textbf{1.00}$ \tiny $\pm .01$ & $0.86$ \tiny $\pm .03$ & $\textbf{1.00}$ \tiny $\pm .01$  \\
Deterministic IL & $0.91$ \tiny $\pm .02$ & $\textbf{0.97}$ \tiny $\pm .02$ & $0.41$ \tiny $\pm .03$ & $\textbf{1.00}$ \tiny $\pm .01$  \\
Stochastic IL & $0.90$ \tiny $\pm .04$ & $0.94$ \tiny $\pm .03$ & $0.34$ \tiny $\pm .05$ & $\textbf{1.00}$ \tiny $\pm .01$  \\
Pure Exploration & $0.09$ \tiny $\pm .14$ & $0.53$ \tiny $\pm .15$ & $0.01$ \tiny $\pm .04$ & $\textbf{1.00}$ \tiny $\pm .01$  \\
Constrained RL & $0.13$ \tiny $\pm .03$ & $0.76$ \tiny $\pm .15$ & $0.40$ \tiny $\pm .03$ & $\textbf{0.96}$ \tiny $\pm .03$  \\
Robust RL & $0.19$ \tiny $\pm .03$ & $0.96$ \tiny $\pm .02$ & $0.47$ \tiny $\pm .02$ & $\textbf{1.00}$ \tiny $\pm .01$  \\
\midrule
Average & 0.58 & 0.88 & 0.48 & \textbf{0.99}  \\
\bottomrule
\end{tabular}
\end{table}

\subsection{Linear RL on DMC}

We report numerical results from Figure \ref{fig:mujoco} (i.e., the evaluation of linear objectives in DMC) in Table \ref{tab:mujoco_linear}.
\begin{table}[h]
\centering
\caption{Episodic returns and step-wise action entropy across different environments as a function of $\|z\|$, averaged over tasks and reported from Figure \ref{fig:mujoco}.}
\label{tab:mujoco_linear}
\begin{tabular}{c c c c c c c c c}
\toprule
\multirow{2}{*}{$\|z\|$} & \multicolumn{2}{c}{walker} & \multicolumn{2}{c}{cheetah} & \multicolumn{2}{c}{quadruped} & \multicolumn{2}{c}{maze} \\
\cmidrule(lr){2-3} \cmidrule(lr){4-5} \cmidrule(lr){6-7} \cmidrule(lr){8-9}
 & Return & Entropy & Return & Entropy & Return & Entropy & Return & Entropy \\
\midrule 0.00& $123$ \tiny $\pm 24$ & $3.5$ \tiny $\pm 0.5$ & $77$ \tiny $\pm 5$ & $14.5$ \tiny $\pm 0.2$ & $171$ \tiny $\pm 10$ & $29.8$ \tiny $\pm 0.2$ & $0$ \tiny $\pm 0$ & $4.6$ \tiny $\pm 0.1$  \\
 0.25& $622$ \tiny $\pm 16$ & $0.3$ \tiny $\pm 0.1$ & $313$ \tiny $\pm 24$ & $13.1$ \tiny $\pm 0.1$ & $681$ \tiny $\pm 2$ & $24.8$ \tiny $\pm 0.3$ & $224$ \tiny $\pm 11$ & $3.8$ \tiny $\pm 0.2$  \\
 0.50& $688$ \tiny $\pm 13$ & $-2.3$ \tiny $\pm 0.2$ & $399$ \tiny $\pm 23$ & $11.2$ \tiny $\pm 0.2$ & $696$ \tiny $\pm 7$ & $20.3$ \tiny $\pm 0.4$ & $350$ \tiny $\pm 62$ & $2.6$ \tiny $\pm 0.2$  \\
 0.75& $698$ \tiny $\pm 12$ & $-8.2$ \tiny $\pm 0.7$ & $440$ \tiny $\pm 12$ & $7.0$ \tiny $\pm 0.5$ & $687$ \tiny $\pm 4$ & $16.1$ \tiny $\pm 0.3$ & $472$ \tiny $\pm 52$ & $-0.0$ \tiny $\pm 0.2$  \\
 1.00& $702$ \tiny $\pm 8$ & $-12.4$ \tiny $\pm 0.7$ & $446$ \tiny $\pm 19$ & $-3.2$ \tiny $\pm 0.9$ & $668$ \tiny $\pm 6$ & $7.1$ \tiny $\pm 1.0$ & $522$ \tiny $\pm 18$ & $-2.9$ \tiny $\pm 0.2$  \\
\bottomrule
\end{tabular}%
\end{table}

\subsection{Visualization of Action Samples}
\label{app:collapse}
While \method retrieves stochastic policies, the standard forward-backward objectives do not explicitly prevent policies to collapse towards determinism. While set of policies retrieved by FB may, in principle, also include stochastic policies, we find that this is, in practice, not the case. As a didactic example, we repeat the visualization in Figure \ref{fig:qualitative}, this time sampling actions from policies trained with FB. These samples are displayed in Figure \ref{fig:collapse}, which confirms that retrieved policies are deterministic.

\begin{figure}
    \centering
    \includegraphics[width=\linewidth]{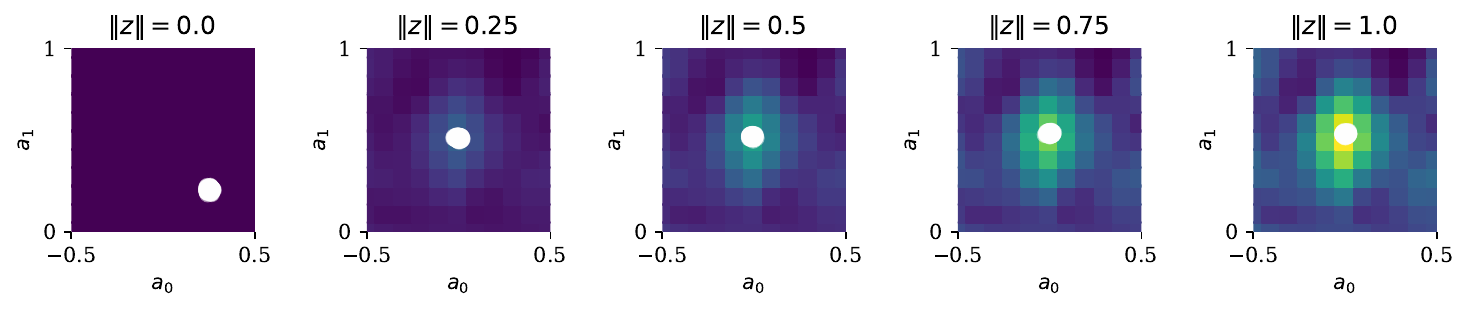}
    \vspace{-6mm}
    \caption{Qualitative evaluation of FB in a didactic environment. White dots are samples from policies $\pi_z$ over a 2D actions space, and the color map represents learned unregularized Q-values $Q_R^z$ for each action ($F_\theta(s_0, a, z)^\top z$). From left to right, we infer task embeddings $z$ solving a goal-reaching task, and scale them linearly. In this case, all retrieved policies are near-deterministic; interestingly, FB often outputs optimal actions for embeddings of norm $\|z\|<1$, despite not being explicitly trained on them.}
    \label{fig:collapse}
\end{figure}

\subsection{Extended results on DMC}
\label{app:implicit}

This Section extends the experimental results summarized in Table \ref{tab:mujoco} in several ways.

First, we include two additional objective ($J_{\text{robust}}$ and $J_{\text{constrained}}$, thus evaluating the same family of non-linear objectives considered in Table \ref{tab:quantitative}. When interpreting results, we remark that the domains considered are not generally designed for non-linear objectives; as dynamics are unknown, it is not clear whether some of these objectives may still admit deterministic optimal policies: for instance, high state entropy may be achievable in \texttt{maze} by a deterministic policy which covers the room in a non-overlapping zig-zag pattern.
We find that Soft FB is consistently the strongest method for the two objectives that do not admit an optimal deterministic policy, namely maximum entropy over states and actions ($\mathcal{H}(\mathrm{M}^\pi)$) and imitation of a stochastic policy ($-KL(\mathrm{M}^\pi;\mathrm{M}^{\pi_{stoch}})$).
Perhaps surprisingly, Soft FB also performs well in imitating a deterministic policy: while KL divergence to the expert is minimized by a deterministic policy, we hypothesize that finding a stochastic policy with a low divergence is generally easier.
For other objectives (e.g., state-only entropy maximization, constrained or robust objectives), we find that the best policy captured by FB may perform as well, confirming that deterministic policies may be sufficient for nonlinear objectives in specific MDPs.

Second, Table \ref{tab:implicit} includes the performance of FB and Soft FB with implicit measure models (2nd and 4th columns, respectively). We find that implicit measure models also perform relatively well, and not significantly worse than explicit measure models in most tasks when combined with Soft FB. We believe that with further compute allocation, explicit measure models can be further improved, as discussed in Appendix \ref{app:implementation}.

Finally, we also report the performance of a random policy sampled from those trained by FB for each objective and domain (first column of Table \ref{tab:implicit}. While this simple baseline is surpassed by Soft FB$_{flow}$ in each instance, the performance difference is not significant in a few cases. This suggests that, for specific objectives, Soft FB may in practice not recover a rich enough class of policies, such that the best of its policies may not be much better than the average FB policy (we remark that the guarantees in Section \ref{sec:enough} only hold with exact minimization of the objectives, and for sufficiently high-dimensional representations). Furthermore, the accuracy of the inference procedure remains limited by the quality of measure models, which are also learned, and therefore imperfect. This is in particular the case for higher-dimensional domains such as \texttt{quadruped}.
\begin{table}[t]
\caption{Zero-shot performance on general utilities in DMC. We extend Table \ref{tab:mujoco} by including three additional methods and two objectives. Scores are averages over 5 seeds, with 95\% confidence intervals and bold numbers signaling overlaps.}
\label{tab:implicit}
\begin{center}
\begin{small}
\begin{tabular}{llccccc}
\toprule
& & FB$_{rand}$ & FB & FB$_{flow}$ & Soft FB & Soft FB$_{flow}$ \\
\midrule
\multirow{6}{*}{\rotatebox[origin=c]{90}{\texttt{walker}}} 
& $\mathcal{H}(\mathrm{M}_\mathcal{S}^\pi)$ & $\textbf{12.99}$ \tiny $\pm .94$ & $\textbf{12.82}$ \tiny $\pm .04$ & $12.56$ \tiny $\pm .91$ & $\textbf{13.91}$ \tiny $\pm .20$ & $\textbf{13.97}$ \tiny $\pm .25$ \\
& $\mathcal{H}(\mathrm{M}^\pi)$ & $9.90$ \tiny $\pm .11$ & $9.77$ \tiny $\pm .09$ & $9.86$ \tiny $\pm .10$ & $\textbf{18.00}$ \tiny $\pm .02$ & $\textbf{18.01}$ \tiny $\pm .03$ \\
& $-KL(\mathrm{M}_\mathcal{S}^\pi; M^{\pi_{stoch}}_\mathcal{S})$ & $\textbf{-6.08}$ \tiny $\pm .89$ & $-6.91$ \tiny $\pm .03$ & $\textbf{-5.35}$ \tiny $\pm .24$ & $\textbf{-5.11}$ \tiny $\pm .27$ & $\textbf{-5.25}$ \tiny $\pm .14$ \\
& $-KL(\mathrm{M}^\pi; M^{\pi_{stoch}})$ & $-9.64$ \tiny $\pm .20$ & $-9.56$ \tiny $\pm .23$ & $-9.40$ \tiny $\pm .17$ & $\textbf{-1.39}$ \tiny $\pm .05$ & $\textbf{-1.48}$ \tiny $\pm .11$ \\
& $-KL(\mathrm{M}_\mathcal{S}^\pi; M^{\pi_{det}}_\mathcal{S})$ & $\textbf{-6.18}$ \tiny $\pm .05$ & $-6.87$ \tiny $\pm .97$ & $-5.69$ \tiny $\pm .27$ & $\textbf{-5.02}$ \tiny $\pm .16$ & $\textbf{-5.69}$ \tiny $\pm .02$ \\
& $-KL(\mathrm{M}^\pi; M^{\pi_{det}})$ & $-9.66$ \tiny $\pm .15$ & $-9.62$ \tiny $\pm .20$ & $-9.43$ \tiny $\pm .18$ & $\textbf{-1.48}$ \tiny $\pm .08$ & $\textbf{-1.53}$ \tiny $\pm .09$ \\
& $J_{\text{robust}}(\mathrm{M}^\pi)$ & $0.39$ \tiny $\pm .04$ & $\textbf{0.51}$ \tiny $\pm .08$ & $\textbf{0.51}$ \tiny $\pm .03$ & $\textbf{0.56}$ \tiny $\pm .12$ & $\textbf{0.44}$ \tiny $\pm .07$ \\
& $J_{\text{constrained}}(\mathrm{M}^\pi)$ & $\textbf{0.13}$ \tiny $\pm .22$ & $\textbf{0.13}$ \tiny $\pm .30$ & $\textbf{0.23}$ \tiny $\pm .32$ & $\textbf{-0.10}$ \tiny $\pm .25$ & $\textbf{0.17}$ \tiny $\pm .19$ \\
\midrule
\multirow{6}{*}{\rotatebox[origin=c]{90}{\texttt{cheetah}}} 
& $\mathcal{H}(\mathrm{M}_\mathcal{S}^\pi)$ & $11.57$ \tiny $\pm .43$ & $12.66$ \tiny $\pm .15$ & $11.74$ \tiny $\pm .43$ & $\textbf{13.69}$ \tiny $\pm .15$ & $\textbf{13.63}$ \tiny $\pm .10$ \\
& $\mathcal{H}(\mathrm{M}^\pi)$ & $9.96$ \tiny $\pm .21$ & $10.59$ \tiny $\pm .88$ & $11.68$ \tiny $\pm .89$ & $\textbf{17.85}$ \tiny $\pm .04$ & $\textbf{17.83}$ \tiny $\pm .07$ \\
& $-KL(\mathrm{M}_\mathcal{S}^\pi; \mathrm{M}^{\pi_{stoch}}_\mathcal{S})$ & $-6.54$ \tiny $\pm .68$ & $-5.58$ \tiny $\pm .20$ & $-5.56$ \tiny $\pm .31$ & $\textbf{-4.36}$ \tiny $\pm .15$ & $\textbf{-4.56}$ \tiny $\pm .31$ \\
& $-KL(\mathrm{M}^\pi; \mathrm{M}^{\pi_{stoch}})$ & $-8.64$ \tiny $\pm .45$ & $-7.90$ \tiny $\pm .90$ & $-7.25$ \tiny $\pm .76$ & $\textbf{-1.11}$ \tiny $\pm .05$ & $\textbf{-1.02}$ \tiny $\pm .08$ \\
& $-KL(\mathrm{M}_\mathcal{S}^\pi; \mathrm{M}^{\pi_{det}}_\mathcal{S})$ & $-6.52$ \tiny $\pm .78$ & $-5.58$ \tiny $\pm .28$ & $-5.45$ \tiny $\pm .65$ & $\textbf{-4.35}$ \tiny $\pm .27$ & $\textbf{-4.38}$ \tiny $\pm .18$ \\
& $-KL(\mathrm{M}^\pi; \mathrm{M}^{\pi_{det}})$ & $-8.92$ \tiny $\pm .40$ & $-8.14$ \tiny $\pm .85$ & $-7.55$ \tiny $\pm .79$ & $\textbf{-1.19}$ \tiny $\pm .05$ & $\textbf{-1.23}$ \tiny $\pm .08$ \\
& $J_{\text{robust}}(\mathrm{M}^\pi)$ & $\textbf{0.19}$ \tiny $\pm .07$ & $\textbf{0.22}$ \tiny $\pm .04$ & $\textbf{0.24}$ \tiny $\pm .06$ & $\textbf{0.24}$ \tiny $\pm .04$ & $\textbf{0.21}$ \tiny $\pm .01$ \\
& $J_{\text{constrained}}(\mathrm{M}^\pi)$ & $\textbf{0.11}$ \tiny $\pm .08$ & $\textbf{0.18}$ \tiny $\pm .09$ & $\textbf{0.14}$ \tiny $\pm .18$ & $\textbf{0.19}$ \tiny $\pm .10$ & $\textbf{0.23}$ \tiny $\pm .07$ \\
\midrule
\multirow{6}{*}{\rotatebox[origin=c]{90}{\texttt{quadruped}}} 
& $\mathcal{H}(\mathrm{M}_\mathcal{S}^\pi)$ & $\textbf{13.84}$ \tiny $\pm .49$ & $\textbf{13.67}$ \tiny $\pm .74$ & $13.37$ \tiny $\pm .43$ & $\textbf{14.48}$ \tiny $\pm .23$ & $\textbf{14.43}$ \tiny $\pm .24$ \\
& $\mathcal{H}(\mathrm{M}^\pi)$ & $10.19$ \tiny $\pm .08$ & $10.18$ \tiny $\pm .14$ & $10.11$ \tiny $\pm .08$ & $\textbf{19.09}$ \tiny $\pm .03$ & $\textbf{19.10}$ \tiny $\pm .01$ \\
& $-KL(\mathrm{M}_\mathcal{S}^\pi; \mathrm{M}^{\pi_{stoch}}_\mathcal{S})$ & $-5.54$ \tiny $\pm .47$ & $-5.99$ \tiny $\pm .60$ & $-5.44$ \tiny $\pm .28$ & $\textbf{-4.95}$ \tiny $\pm .35$ & $\textbf{-4.64}$ \tiny $\pm .20$ \\
& $-KL(\mathrm{M}^\pi; \mathrm{M}^{\pi_{stoch}})$ & $-9.93$ \tiny $\pm .08$ & $-9.87$ \tiny $\pm .18$ & $-9.92$ \tiny $\pm .16$ & $\textbf{-0.75}$ \tiny $\pm .02$ & $\textbf{-0.76}$ \tiny $\pm .04$ \\
& $-KL(\mathrm{M}_\mathcal{S}^\pi; \mathrm{M}^{\pi_{det}}_\mathcal{S})$ & $-5.64$ \tiny $\pm .53$ & $-6.07$ \tiny $\pm .74$ & $-5.61$ \tiny $\pm .30$ & $-4.96$ \tiny $\pm .14$ & $\textbf{-4.52}$ \tiny $\pm .16$ \\
& $-KL(\mathrm{M}^\pi; \mathrm{M}^{\pi_{det}})$ & $-9.95$ \tiny $\pm .07$ & $-9.99$ \tiny $\pm .22$ & $-9.94$ \tiny $\pm .20$ & $\textbf{-0.79}$ \tiny $\pm .02$ & $\textbf{-0.79}$ \tiny $\pm .02$ \\
& $J_{\text{robust}}(\mathrm{M}^\pi)$ & $0.07$ \tiny $\pm .01$ & $\textbf{0.14}$ \tiny $\pm .02$ & $0.09$ \tiny $\pm .03$ & $\textbf{0.17}$ \tiny $\pm .04$ & $0.08$ \tiny $\pm .01$ \\
& $J_{\text{constrained}}(\mathrm{M}^\pi)$ & $\textbf{0.03}$ \tiny $\pm .03$ & $\textbf{0.02}$ \tiny $\pm .04$ & $\textbf{0.04}$ \tiny $\pm .06$ & $\textbf{0.06}$ \tiny $\pm .04$ & $\textbf{0.07}$ \tiny $\pm .06$ \\
\midrule
\multirow{6}{*}{\rotatebox[origin=c]{90}{\texttt{maze}}} 
& $\mathcal{H}(\mathrm{M}_\mathcal{S}^\pi)$ & $\textbf{10.31}$ \tiny $\pm .34$ & $\textbf{10.68}$ \tiny $\pm .21$ & $\textbf{10.83}$ \tiny $\pm .56$ & $\textbf{10.62}$ \tiny $\pm .66$ & $\textbf{11.22}$ \tiny $\pm .70$ \\
& $\mathcal{H}(\mathrm{M}^\pi)$ & $10.00$ \tiny $\pm .48$ & $10.96$ \tiny $\pm .70$ & $10.52$ \tiny $\pm .66$ & $\textbf{15.67}$ \tiny $\pm .08$ & $\textbf{15.54}$ \tiny $\pm .16$ \\
& $-KL(\mathrm{M}_\mathcal{S}^\pi; \mathrm{M}^{\pi_{stoch}}_\mathcal{S})$ & $-6.32$ \tiny $\pm .62$ & $\textbf{-6.15}$ \tiny $\pm .67$ & $-6.49$ \tiny $\pm .36$ & $\textbf{-5.23}$ \tiny $\pm .31$ & $\textbf{-5.25}$ \tiny $\pm .90$ \\
& $-KL(\mathrm{M}^\pi; \mathrm{M}^{\pi_{stoch}})$ & $-7.22$ \tiny $\pm .34$ & $-6.61$ \tiny $\pm .14$ & $-6.45$ \tiny $\pm .61$ & $\textbf{-1.08}$ \tiny $\pm .21$ & $\textbf{-1.39}$ \tiny $\pm .31$ \\
& $-KL(\mathrm{M}_\mathcal{S}^\pi; \mathrm{M}^{\pi_{det}}_\mathcal{S})$ & $-6.62$ \tiny $\pm .54$ & $-6.42$ \tiny $\pm .58$ & $-6.57$ \tiny $\pm .61$ & $\textbf{-5.76}$ \tiny $\pm .32$ & $\textbf{-5.07}$ \tiny $\pm .69$ \\
& $-KL(\mathrm{M}^\pi; \mathrm{M}^{\pi_{det}})$ & $-8.67$ \tiny $\pm .54$ & $-7.27$ \tiny $\pm .96$ & $-7.27$ \tiny $\pm .58$ & $\textbf{-2.37}$ \tiny $\pm .27$ & $\textbf{-2.51}$ \tiny $\pm .35$ \\
& $J_{\text{robust}}(\mathrm{M}^\pi)$ & $0.17$ \tiny $\pm .10$ & $\textbf{0.44}$ \tiny $\pm .06$ & $\textbf{0.34}$ \tiny $\pm .19$ & $0.33$ \tiny $\pm .13$ & $\textbf{0.52}$ \tiny $\pm .05$ \\
& $J_{\text{constrained}}(\mathrm{M}^\pi)$ & $-0.01$ \tiny $\pm .11$ & $\textbf{0.46}$ \tiny $\pm .22$ & $\textbf{0.53}$ \tiny $\pm .24$ & $\textbf{0.31}$ \tiny $\pm .22$ & $\textbf{0.57}$ \tiny $\pm .14$ \\
\bottomrule
\end{tabular}
\end{small}
\end{center}
\end{table}

\section{Implementation Details}
\label{app:implementation}

Our practical implementation of Soft FB builds upon the vanilla FB algorithm introduced in \citet{touati2021learning} and open-sourced in \citet{tirinzoni2024zero}.
Architecturally, the sole modification lies in changing the policy's parameterization: from a Gaussian with a fixed standard deviation of $\sigma=0.3$ to a squashed Gaussian with learnable standard deviation. This modification is applied to all of the algorithms we evaluate. All networks are MLPs with embedding layers; width and depth are reported in Table \ref{tab:hyperparameters}.

Unlike FB, Soft FB trains an action-state value function to estimate discounted sums of entropy terms (see Equation \ref{eq:entropy_td}), which shares the same optimization hyperparameters and the same architecture as the forward map $F_\theta$, but outputs scalar values instead of embeddings. As the forward map, the critic also relies on twin networks, and on a target network updated at the same pace.
While the discounted sum of entropies may also be estimated implicitly through forward and backward representations, we found that training a separate critic resulted in more reliable estimates.

When explicit measure models are used for inference, a flow-matching generative model is trained in parallel with Soft FB. For this purpose, we closely follow the implementation details described in \citep{farebrother2025temporal}. Due to limited computational budget, we train the model for fewer gradient steps (3M instead of 8M). While this was sufficient for policy selection, allocating further resource to measure estimation may bring additional performance gains.

Our implementation of Soft FB inherits all of FB's hyperparameters, without any additional tuning; we present the main hyperparameters in Table \ref{tab:hyperparameters}.
Soft FB does not introduce any additional hyperparameter: the standard entropy regularization coefficient common in entropy-regularized RL \citep{haarnoja2017reinforcement} is not needed, as Soft FB trains on all levels of entropy regularization, as discussed in Section \ref{sec:practical}.
Additionally scaling the entropy term by a fixed coefficient is possible, and may lead to better performance for tasks in which high- or low-entropy policies are favored.


\begin{table}[htbp]
    \centering
    \caption{Hyperparameter Configuration}
    \label{tab:hyperparameters}
    \begin{tabular}{lc}
        \toprule
        \textbf{Parameter} & \textbf{Value} \\
        \midrule
        $F_\theta$ - Learning rate & $1 \times 10^{-4}$ \\
        $F_\theta$ - Width & 1024 \\
        $F_\theta$ - Depth & 4 \\
        $B_\phi$ - Learning rate & $1 \times 10^{-4}$ \\
        $B_\phi$ - Width & 256 \\
        $B_\phi$ - Depth & 3 \\
        $\pi_\psi$ - Learning rate & $1 \times 10^{-4}$ \\
        $\pi_\psi$ - Width & 1024 \\
        $\pi_\psi$ - Depth & 4 \\
        $Q_{\mathcal{H}, \eta}$ - Learning rate & $1 \times 10^{-4}$ \\
        $Q_{\mathcal{H}, \eta}$ - Width & 1024 \\
        $Q_{\mathcal{H}, \eta}$ - Depth & 4 \\
        Optimizer & Adam \citep{kingma2015adam} \\
        Dimensionality of $z$ & $50$ \\
        Polyak coefficient $\tau$ & $0.01$ \\
        Orthonormality regularization coefficient & $1.0$ \\
        Goal sampling ratio & $0.5$ \\
        Q-value aggregation & \texttt{mean} \\
        Batch size & $1024$ \\
        \bottomrule
    \end{tabular}
\end{table}

\subsection{Low-dimensional Evaluation}

The didactic MDP described in Section \ref{subsec:qualitative} has a discount factor of $0.5$; due to its simplicity, agents are trained for $5\cdot10^4$ gradient steps.
The task inference procedure is carried out as described in Section \ref{sec:inference} with $1024$ uniformly sampled reward embeddings as candidates. In the case of standard FB, they are always sampled from the surface of an hypersphere, as the agent was not trained for embeddings lying within its volume.
For each $z$, $1024$ samples are drawn from the $z$-conditional successor measure $\hat M^{\pi_z}$, estimated through either implicit or explicit measure models.
These samples are then evaluated according the objective, which are presented in full in Table \ref{tab:objectives}; the highest-scoring sample is provided as a conditioning to the policy, which is then evaluated in the environment for $1024$ episodes to produce the scores reported in Table \ref{tab:quantitative}.

Objectives were designed to be easily interpretable: the linear task encourages states on a unit circle around the origin, the goal-conditioned task involves landing in proximity of a certain state, while the deterministic and stochastic imitation learning tasks feature an expert which always visits the same state, or visits states uniformly on a line, respectively.
The pure exploration objective simply maximizes state coverage, while the robust objective considers the minimum between a goal conditioned objective and its complement, and the constrained objective maximizes a goal-conditioned objective within a given range. For objectives requiring estimation of entropies or Kullback-Leibler divergence, we employ a standard nearest-neighbor-based estimator \citep{wang2009divergence}, with $k=3$. For each objective, we normalized scores between minima and maxima (theoretical if available, or estimated from the policies with lowest performance), as described in the last two columns of Table \ref{tab:objectives}. Across this work, confidence intervals are computed assuming scores are normally distributed.

\begin{table}[htbp]
\caption{Objectives considered in the quantitative evaluation in low-dimensional settings (see Figure \ref{fig:quantitative}). $s=(x, y)$ denotes the two coordinates of each state $s \in \mathcal{S}$.}
\centering
\begin{tabular}{llll}
\toprule
\textbf{Task} & \textbf{Objective} & \texttt{min} & \texttt{max}\\
\midrule
Linear RL & $\max_\pi \langle \mathrm{M}_\mathcal{S}^\pi , R \rangle$ with $R(x, y) = - ((x^2+y^2) - 1)$ & $\approx 0$ & $1$ \\
Goal-conditioned RL & $\max_\pi \langle \mathrm{M}_\mathcal{S}^\pi , R_{(0.0, 0.5)}(x,y) \rangle$ with $R_{(x_g, y_g)}(x, y) = \mathbf{1}_{(x-x_g)^2+(y-y_g)^2 < 0.2}$ & $0$ & $1$ \\
Deterministic IL & $\max_\pi -D_\text{KL}(\mathrm{M}_\mathcal{S}^\pi||\mathrm{M}_\mathcal{S}^\star)$ with $\mathrm{M}_\mathcal{S}^\star = \mathbf{1}_{s=(0.0, 0.5)}$ & $\approx -15$ & $0$ \\
Stochastic IL & $\max_\pi -D_\text{KL}(\mathrm{M}_\mathcal{S}^\pi||\mathrm{M}_\mathcal{S}^\star)$ with $\mathrm{M}_\mathcal{S}^\star = \mathcal{U}(\{x \in [-0.5, 0.5], y=0]\})$ & $\approx -15$ & $0$ \\
Pure Exploration & $\max_\pi \mathcal{H(\mathrm{M}_\mathcal{S}^\pi)}$ & $\approx 8$ & $\approx 16$ \\
Robust MDP & $\max_\pi \min \big (\langle \mathrm{M}_\mathcal{S}^\pi, R_{(0.0, 0.5)}(x,y)\rangle, \langle \mathrm{M}_\mathcal{S}^\pi, 1-R_{(0.0, 0.5)}(x,y)\rangle \big)$ & $0$ & $0.5$ \\
Constrained RL & $\max_\pi \langle \mathrm{M}_\mathcal{S}^\pi, R_{(0, 0.5)}(x,y)\rangle$ s.t. $\langle \mathrm{M}_\mathcal{S}^\pi, R_{(0., 0.5)}(x,y) \rangle < 0.9$ & $0$ & $0.9$ \\
\bottomrule
\end{tabular}
\label{tab:objectives}
\end{table}

\subsection{High-dimensional evaluation}

\begin{figure}[htbp]
    \centering
    \begin{minipage}{0.23\textwidth}
        \centering
        \texttt{walker} \\
        \vspace{2mm}
        \includegraphics[width=\textwidth]{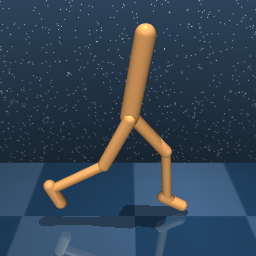}
    \end{minipage}
    \hfill
    \begin{minipage}{0.23\textwidth}
        \centering
        \texttt{cheetah} \\
        \vspace{2mm}
        \includegraphics[width=\textwidth]{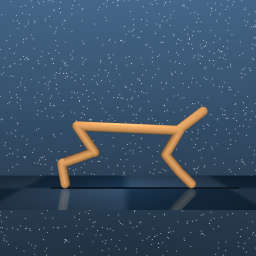}
    \end{minipage}
    \hfill
    \begin{minipage}{0.23\textwidth}
        \centering
        \texttt{quadruped} \\
        \vspace{2mm}
        \includegraphics[width=\textwidth]{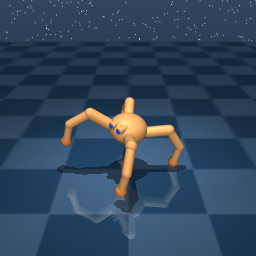}
    \end{minipage}
    \hfill
    \begin{minipage}{0.23\textwidth}
        \centering
        \texttt{maze} \\
        \vspace{2mm}
        \includegraphics[width=\textwidth]{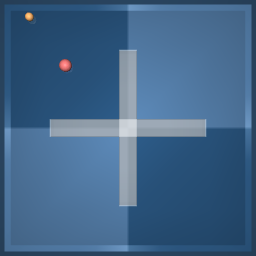}
    \end{minipage}

    \caption{The four domains considered from ExORL \citep{yarats2022exorl} and the Deepmind Control Suite \citep{tassa2018deepmind}.}
    \label{fig:environments}
\end{figure}

High-dimensional experiments revolve around locomotion and navigation tasks from the Deepmind Control Suite \citep{tassa2018deepmind} and use ExORL data \citep{yarats2022exorl}. All evaluation parameters, including linear reward functions, are taken from \citet{touati2021learning}: algorithms are trained until convergence (for 3M gradient steps) with a discount factor of $0.98$ ($0.99$ for \texttt{maze}).
The inference procedure for General Utilities differs minimally from the didactic environment (detailed previously): it evaluates $1024$ candidates for $z$, but considers slightly more samples from the estimated successor measure ($2048$). The chosen policy is then evaluated through a Monte Carlo estimator: we execute the policy for $10$ episodes and resample visited state-action pairs according to a geometric distribution parameterized by $1-\gamma$.

Objectives involving entropy or Kullback-Leibler divergences also rely on a K-nearest-neighbor-based estimator, with $k=3$ \citep{wang2009divergence}.
For imitation learning objectives, the expert is a deterministic optimal policy for one of the tasks defined by the suite (\texttt{walk} for all domains, except for \texttt{maze}, which demonstrates \texttt{reach\_bottom\_left}); the stochastic variants simply inject Gaussian noise ($\sigma=1.0$) to expert actions.
The robust and constrained objectives are less direct in their definition. For simplicity, we consider the velocities of the first two degrees of freedoms for all embodiments (e.g., going right/left and up/down for \texttt{maze}), and refer to the function extracting them from a state as as $v_x$ and $v_y$. The robust objective is simply the minimum between the discounted averages of the two absolute velocities, i.e. $J_\text{robust}(\mathrm{M}^\pi) = \min (\langle \mathrm{M}^\pi,R_x  \rangle, \langle \mathrm{M}^\pi,R_y \rangle)$, where $R_x(s) = \text{abs}(v_x(s))$ and $R_y(s) = \text{abs}(v_y(s))$. The constrained objective encourages high velocity in the first degree of freedom, while ensuring the velocity is not always positive: $J_\text{constrained}(\mathrm{M}^\pi) = \langle  \mathrm{M}^\pi, R\rangle$ constrained by $\langle  \mathrm{M}^\pi, R_c \rangle > 0$ with $R(s)=v_x(s)$ and $R_c(s)=\mathbf{1}_{v_x(s) < 0}$.

As in this case minimum and maximum performance are not easily estimated, we do not normalize scores.

\end{document}